\documentclass[10pt,twocolumn,letterpaper]{article}

\usepackage{cvpr}              

\usepackage{algorithm}
\usepackage{algpseudocode}
\usepackage{amsfonts}
\usepackage{amsmath}
\usepackage{amssymb}
\usepackage{amsthm}
\usepackage[accsupp]{axessibility}
\usepackage{bm}
\usepackage{multirow}
\usepackage{nameref}
\usepackage{subcaption}
\usepackage{thmtools}

\newtheorem{asm}{Assumption}
\newtheorem{prop}{Proposition}
\newtheorem{thm}{Theorem}

\def\eqref#1{equation~\ref{#1}}

\def\1{\bm{1}}

\def\vzero{{\bm{0}}}

\def\vtheta{{\bm{\theta}}}

\def\vb{{\bm{b}}}
\def\vc{{\bm{c}}}

\def\vu{{\bm{u}}}
\def\vv{{\bm{v}}}

\def\vx{{\bm{x}}}
\def\vy{{\bm{y}}}

\def\mA{{\bm{A}}}

\def\mF{{\bm{F}}}

\def\mH{{\bm{H}}}
\def\mI{{\bm{I}}}

\def\mS{{\bm{S}}}

\def\mU{{\bm{U}}}
\def\mV{{\bm{V}}}

\def\mX{{\bm{X}}}

\DeclareMathAlphabet{\mathsfit}{\encodingdefault}{\sfdefault}{m}{sl}
\SetMathAlphabet{\mathsfit}{bold}{\encodingdefault}{\sfdefault}{bx}{n}
\newcommand{\tens}[1]{\bm{\mathsfit{#1}}}

\def\tG{{\tens{G}}}

\def\gB{{\mathcal{B}}}

\def\gD{{\mathcal{D}}}

\def\gF{{\mathcal{F}}}

\def\gJ{{\mathcal{J}}}

\def\gL{{\mathcal{L}}}
\def\gM{{\mathcal{M}}}

\def\gO{{\mathcal{O}}}

\def\gR{{\mathcal{R}}}

\def\gV{{\mathcal{V}}}

\def\sC{{\mathbb{C}}}

\newcommand{\E}{\mathbb{E}}

\usepackage{bibunits}
\defaultbibliographystyle{ieeenat_fullname}

\definecolor{cvprblue}{rgb}{0.21,0.49,0.74}
\usepackage[breaklinks,colorlinks,allcolors=cvprblue]{hyperref}

\AtEndPreamble{%
  \crefname{asm}{Assumption}{Assumptions}%
  \Crefname{asm}{Assumption}{Assumptions}%
  \crefname{prop}{Proposition}{Propositions}%
  \Crefname{prop}{Proposition}{Propositions}%
  \crefname{thm}{Theorem}{Theorems}%
  \Crefname{thm}{Theorem}{Theorems}%
}

\def\paperID{34018} 
\def\confName{CVPR}
\def\confYear{2026}

\title{%
    Model Merging on Loss Landscape: A Geometry Perspective
} %

\author{
    Juanwu Lu$^{\dagger}$\thanks{Work done as an intern at Waymo LLC.} \quad
    Anand Bhaskar$^{\ddagger}$ \quad
    Brian Axelrod$^{\ddagger}$ \quad
    Ekaterina Tolstaya$^{\ddagger}$ \quad
    Tristan Emrich$^{\ddagger}$
    \\
    $^{\dagger}$~Purdue University\quad
    $^{\ddagger}$~Waymo LLC
    \\
    {\tt\small juanwu@purdue.edu, \{anand.bhaskar,baxelrod,eig,tristane\}@waymo.com}
}

\begin{document}

\maketitle

\begin{bibunit}
\begin{abstract}

    Model merging offers a promising avenue to knowledge integration and parallel development without retraining. Yet, existing methods either ignore the geometry of the loss landscape or rely on intractable full-space Hessian approximations. We propose \emph{EpiMer}, a framework that casts model merging as solving the Fr\'echet mean on a Riemannian manifold and restricts the computation to a low-rank subspace spanned by the task vectors. With the expected Hessian as the metric, we reveal a connection between local curvature and epistemic uncertainty of the parameters. Our theoretical analysis decomposes the merging error bound into the subspace Fr\'echet variance and the residual energy, and provides a closed-form characterization of \emph{when} curvature-aware merging provably outperforms flat-geometry methods. In addition, our framework unifies both curvature-aware methods and recent spectral methods as special cases of the subspace Fr\'echet mean with different geometric metrics. Merging fine-tuned CLIP-ViT models on eight image classification tasks, Epistemic Merging strictly outperforms the baselines on all three CLIP-ViT backbones at matched rank, improving the across-task average accuracy and worst-task accuracy on every backbone.


\end{abstract}

\section{Introduction}
\label{sec: intro}

With the increasing accessibility of pre-trained models, fine-tuning for adaptation to specialized tasks has become the \emph{de facto} standard for building powerful machine learning systems~\cite{NEURIPS2021_b36ed8a0, Wortsman_2022_CVPR, ding2023parameter}. As the sizes of task-specific models scale up, a new challenge emerges: \emph{consolidation of their diverse capabilities into a single, unified model for both computational and memory efficiency}. Recently, model merging has gained significant interest as a means to effectively and efficiently ensemble multiple models at inference time without costly retraining.

Existing merging strategies primarily operate in the flat, Euclidean parameter space. They seek an optimal merged model by directly solving for the linear average of task-specific models, leveraging auxiliary information such as performance rankings~\cite{pmlr-v162-wortsman22a}, task vectors~\cite{ilharco2023editingmodelstaskarithmetic}, and heuristic-based conflict resolution~\cite{NEURIPS2023_1644c9af}. While often effective, these methods share a fundamental limitation --- \emph{they ignore the geometry of the loss landscape}. As shown in~\cref{fig: illustration}, ignoring curvature in the loss landscape can cause a merged model to locate on global high-loss barriers, leading to destructive interference and potential catastrophic forgetting.

\begin{figure}[!t]
    \centering
    \includegraphics[width=\linewidth]{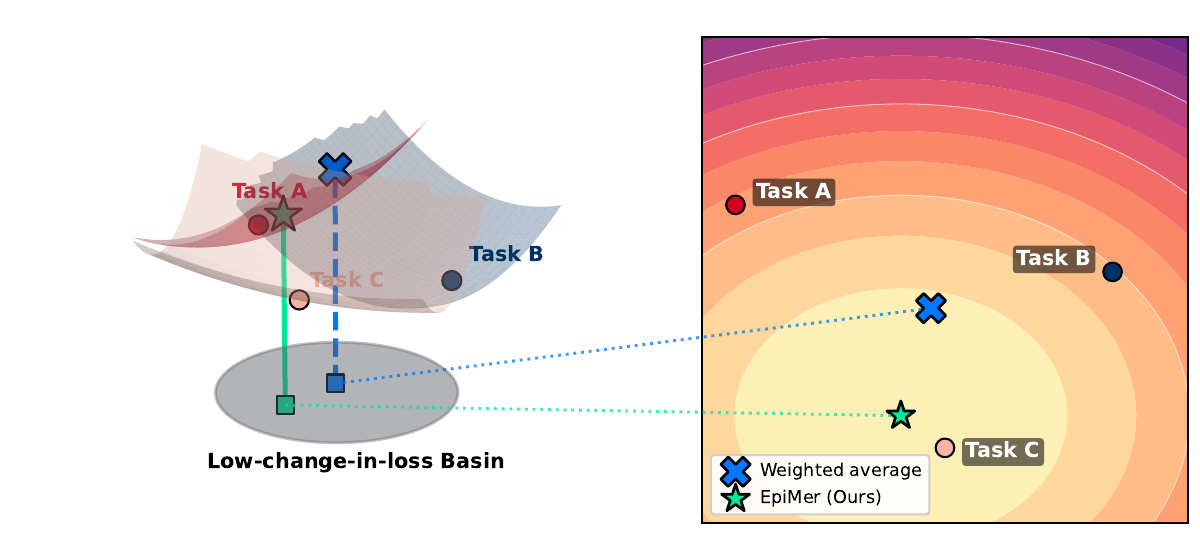}
    \caption{Illustration of the core idea. \textbf{Left:} Three task loss manifolds in a parameter frame anchored at the pre-trained parameters $\vtheta_{0}$. Each task has its minimum at a different location, and the overlap of their basins defines a joint low-loss region. The arithmetic weighted average of the task vectors (\textcolor[HTML]{0078FF}{blue} cross) lands where one manifold's loss remains catastrophically high, whereas our proposed merged model (\textcolor[HTML]{01BB7D}{green} star) sits inside the joint basin with a much lower worst-case loss. \textbf{Right:} The same three losses projected to the low-rank subspace $\mathcal{S}$ spanned by the task vectors $\{\bm{\delta}_{t}\mid{t=1,\ldots,3}\}$. The summed joint-loss heatmap places the weighted average in a high-loss region and our Fr\'echet mean in the central basin.}
    \label{fig: illustration}
    \vspace{-2em}
\end{figure}

Meanwhile, curvature-aware methods that incorporate second-order information~\cite{NEURIPS2022_70c26937, daheim2024modelmerginguncertaintybasedgradient} need to tackle a different challenge. \emph{Computing or approximating the Hessian in the full parameter space is either intractable or so noisy that it nullifies the theoretical advantage over simpler methods}. Recent spectral approaches~\cite{gargiulo2025tasksingularvectorsreducing, marczak2025taskleftbehindisotropic} bypass curvature entirely by operating on SVD subspaces of task vectors, achieving strong empirical results but without geometric justification. This raises a fundamental question: \emph{when does curvature actually matter in model merging, and when is flat-geometry merging sufficient?}

In this paper, we introduce \textbf{Epi}stemic \textbf{Mer}ging (\textbf{EpiMer}), a framework that aims to answer this question both theoretically and practically. We cast model merging as the problem of finding the Fr\'echet mean on a Riemannian landscape whose metric is induced by the expected Hessian of the task-specific losses. Crucially, we restrict the computation to the \emph{low-rank subspace} spanned by the task vectors, where the projected Hessian is sufficiently small to be evaluated exactly. This makes curvature-aware merging both principled and practical.

Our experiments, which merge CLIP Vision Transformer (CLIP-ViT) models~\cite{radford2021learningtransferablevisualmodels} fine-tuned on eight distinct image classification tasks, demonstrate the advantage of our approach against modern baselines. Furthermore, we propose a computable diagnostic method that exploits the curvature heterogeneity $\eta$ to predict \emph{a priori} whether curvature-aware merging will outperform flat-geometry methods. In short, our contributions in this paper are as follows:
\begin{itemize}[leftmargin=1.7em]
    \item [\textbf{C1}] Establish a \textbf{geometric framework} that unifies both curvature-aware methods and spectral methods as instances of the Fr\'echet mean under different subspace and metric choices.
    \item [\textbf{C2}] Conduct \textbf{theoretical analysis} that decomposes the error bound of merging failure into subspace variance and residual energy, and derives a closed-form characterization of when curvature-aware merging provably outperforms flat-geometry alternatives.
    \item [\textbf{C3}] Propose \textbf{Epistemic Merging}, a practical method with curvature-guided subspace selection, that strictly outperforms the strongest flat-geometry baseline TSV-M~\cite{gargiulo2025tasksingularvectorsreducing} on all three CLIP-ViT backbones at matched rank and improves worst-task accuracy.
\end{itemize}

\section{Model Merging}
\label{sec: preliminary}

This section provides a brief introduction to the problem of model merging and our motivations. To promote readability, readers can find a table of notation in~\Cref{appx: notation}.

Let $\vtheta_{0}$ be the vector of parameters of a pre-trained model. Given a set of task datasets $\gD=\left[\gD_{1},\dots\gD_{T}\right]$, we can obtain a collection of fine-tuned models with their parameters denoted as $\left\{\vtheta_{t}\mid{t=1, \ldots, T}\right\}$. Herein, each fine-tuned model minimizes its task-specific loss function $\gL_{t}$ using samples from the task-specific datasets:
\begin{equation}
    \vtheta_{t}\triangleq
    \underset{\vtheta\in\Theta}{\arg\min}\E_{\vx\sim\gD_{t}}\gL_{t}(\vx,\vtheta),
\end{equation}
where $\vx$ denotes a single sample. Model merging aims to solve for a single model that fuses knowledge across all tasks. It is equivalent to solving for a single model that jointly minimizes all task-specific losses given by
\begin{equation}
    \vtheta_{m}\gets\underset{\vtheta\in\Theta}{\arg\min}\sum\limits_{t=1}^{T}\lambda_{t}\E_{\vx\sim\gD_{t}}\gL_{t}(\vx,\vtheta),
    \label{eq: objective}
\end{equation}
where $\lambda_{t}\in\mathbb{R}$ are scalars reflecting the relative importance of each task. However, directly optimizing the objective, which requires extensive training and access to all task-specific datasets, is computationally expensive. A plethora of existing methods~\cite{pmlr-v162-wortsman22a, ilharco2023editingmodelstaskarithmetic, NEURIPS2023_1644c9af} have focused on \emph{training-free} methods and proposed to compute the merged model as an \textbf{average} of task-specific models, given by
\begin{equation}
    \vtheta_{m}\gets\sum_{t=1}^{T}\lambda_{t}\vtheta_{t}\quad\text{or}\quad\vtheta_{m}\gets\vtheta_{0}+\sum\limits_{t=1}^{T}\lambda_{t}(\vtheta_{t}-\vtheta_{0}).
    \label{eq: linear-average}
\end{equation}
We will demonstrate in the following section that these solutions are special cases of a more general framework by assuming that the loss landscapes are~\emph{Euclidean} and taking a single step along the~\textbf{steepest} direction. Despite their effectiveness, we argue that this steepest direction is often~\emph{not the most efficient direction} since it ignores the landscape of task-specific losses, where a slight change in the parameter $\vtheta_{t, i}$ can lead to massive changes in task losses $\gL_{t}$ while a significant change in another parameter $\vtheta_{t,j}$ might have a negligible effect. Such sensitivity to parameter changes is closely associated with the concept of \textbf{epistemic uncertainty} in model parameters.

To this end, our objective in this paper is to establish a theoretical framework that helps understand model merging beyond parameter averaging by incorporating the landscape of task-specific losses. In particular, we want to answer the following research questions:
\begin{itemize}[leftmargin=1.75em]
    \item [\textbf{R1}] What is a more generalized problem setting for model merging that extends parameter averaging by incorporating \emph{landscapes of task losses}?
    \item [\textbf{R2}] \emph{When} does curvature-aware merging provably help, and how can we design \emph{a practical method} that exploits curvature in the subspace where it matters?
    \item [\textbf{R3}] Can we quantify \emph{the ability to merge} ahead of merging?
\end{itemize}
The following section addresses these questions by proposing an alternative problem setting for model merging and introducing a set of new solutions.

\section{Epistemic Merging}
\label{sec: method}

We introduce \textbf{Epi}stemic \textbf{Mer}ging (EpiMer). This framework re-frames model merging from a heuristic-based averaging problem to one of navigating on a degenerate Riemannian manifold of task losses. We begin with formulating the geometric objective (\cref{subsec: meth_objective}), then derive a tractable subspace solver that takes the Fr\'echet mean from the intractable ambient space onto a per-task tagged basis on which the projected Hessian is dense (\cref{subsec: subspace-epimer}), and finally analyze the resulting algorithm with an error bound, a curvature-advantage characterization, and its connections to prior merging methods (\cref{subsec: theory}).

\subsection{Geometric Formulation}
\label{subsec: meth_objective}

Following existing works, we define model merging as an operator on the parameter space $\Theta$. We propose to model such a space as a differentiable \textbf{manifold} $\gM\subset\mathbb{R}^{m}$ endowed with a metric tensor $\tG$, where $m$ is the number of parameters. Herein, each point on the manifold $\vtheta\in\Theta$ corresponds to a model that minimizes a loss $\gL(\vx,\vtheta)$. We make the following assumptions:
\begin{asm}
    \label[assumption]{asm: smooth}
    The task losses are at least twice continuously differentiable with respect to the parameters $$\gL(\vx,\vtheta):\mathbb{R}^{n}\times\mathbb{R}^{m}\to\mathbb{R}\in{\sC^{2}}.$$
\end{asm}
\begin{asm}\label{asm: local-opt}
    Each fine-tuned model $\vtheta_{t}$ is at a local minimum where the expected gradients of loss approach zero $$\mathbb{E}_{\vx\sim\gD_{t}}[\nabla\gL_{t}(\vtheta)]\to{\vzero}.$$
\end{asm}
In addition, we propose to endow the parameter manifold as in statistical manifolds~\cite{rao1945information,amari2016information} with the expected Hessian of the loss function as the metric tensor
\begin{equation}
    \tG(\vtheta)\triangleq\mathbb{E}_{\vx}\left[\nabla^{2}_{\vtheta}\gL(\vx,\vtheta)\right]=\mathbb{E}_{\vx}\left[\mH_{\gL}(\vtheta)\right].
    \label{eq: metric-tensor}
\end{equation}
At a local minimum, the Hessian is \emph{positive semi-definite (PSD)}, making $\tG(\vtheta)$ a valid Riemannian metric. However, since modern deep neural networks are highly over-parameterized~\cite{neyshabur2017geometryoptimizationimplicitregularization,doi:10.1073/pnas.1903070116}, the above metric tensor may be \emph{degenerate}, making the space a \textit{degenerate Riemannian manifold}. On such a manifold, a \textbf{geodesic} is a path $\gamma(t):\left[0,1\right]\to\gM$ that parallel transports its tangent vectors and locally minimizes the squared distance between two sets of parameters given by
\begin{equation}
\begin{aligned}
    &d^{2}_{g}(\vtheta,\vtheta^{\prime})(\gamma)
    \triangleq\int_{0}^{1}\frac{d\gamma(\tau)}{d\tau}
    \cdot\tG\left(\gamma(\tau)\right)\cdot \frac{d\gamma(\tau)}{d\tau}d\tau, \\
    &\text{where }
    \gamma(0)=\vtheta,\quad\gamma(1)=\vtheta^{\prime}.
    \label{eq: distance}
\end{aligned}
\end{equation}
From \eqref{eq: metric-tensor} and \eqref{eq: distance}, a geodesic is essentially a path that minimizes the total change in loss transporting from one set of parameters to another, which preferentially navigates through regions of maximal \textbf{epistemic uncertainty}.

Since we are targeting the merging of models fine-tuned from the exact pre-trained progenitor, we assume the existence of a connected basin of them:
\begin{asm}\label{asm: connected-basin}
    On the loss manifold of task $t$, all fine-tuned models $\{\vtheta_{i}|i=1, \dots, T\}$ reside within a single low-change-in-loss basin with the task model $\vtheta_{t}$. Hence, there exists at least one continuous geodesic $\gamma_{t}$ which travels through all such low-change-in-loss basins.
\end{asm}
Based on this geometric formulation, we re-frame the objective of model merging as the search for a geometric center of mass of all models given by
\begin{equation}
    \vtheta_{m}\triangleq\underset{\vtheta\in\gM,\gamma_{t}}
    {\arg\min}\sum_{t=1}^{T}\lambda_{t}\int_{\gamma_{t}(0)=\vtheta_{t}}^{\gamma_{t}(1)=\vtheta}
    \frac{d\gamma_t(\tau)}{d\tau}\tG_{t}(\gamma(\tau))\frac{d\gamma_{t}(\tau)}{d\tau}d\tau,
    \label{eq: frechet-objective}
\end{equation}
which essentially solves for the \textbf{Fr\'echet Mean} of the task parameters constrained by the geometry of task losses. Under assumptions~\ref{asm: smooth} and~\ref{asm: local-opt}, we make the following proposition. See detailed proof in ~\cref{appx: derivations}.
\begin{prop}\label{prop: equiv-obj}
    Minimizing the multi-task loss is equivalent to finding the Fr\'echet Mean of the models to the second-order approximation.
\end{prop}
To this end, we argue that the rigorous foundation for model merging is to solve for the Fr\'echet mean of all models as a proxy for the ideal multi-task solution in~\eqref{eq: objective}, establishing a connection to the existing areas such as continual learning and transfer learning~\cite{ritter2018onlinestructuredlaplaceapproximations,kirkpatrick2017overcoming,pmlr-v37-martens15,pmlr-v37-schulman15}.
In the following section, we propose a practical algorithm that reduces the optimization to a single linear system over a low-rank subspace shared across tasks.

\begin{table*}[!t]
    \centering
    \caption{Connections to existing model merging methods revealed by the subspace Fr\'echet mean in \cref{eq: affine-sol}. Methods are parameterized by the subspace $\mS$ and the metric $\tilde{\mH}_t$ used within it. Full-space methods ($\mS = \mI_m$) include both flat-geometry (AM, TA, NAN, TIES) and curvature-aware (FA, RegMean, \cite{daheim2024modelmerginguncertaintybasedgradient}) approaches. Spectral methods (TSV-M, Isotropic Merging) use SVD subspaces with identity metrics. EpiMer is the only method that combines a non-trivial subspace with a curvature-aware metric.}
    \begin{tabular}{@{}l|c|c|c|c|l@{}}
    \toprule
    \textbf{Methods} & \textbf{Subspace} $\mS$ & $\lambda_{t}$ & $\vtheta_{\text{init}}$ & \textbf{Metric} $\tilde{\mH}_{t}$ & \textbf{Limitations} \\
    \midrule
    \multicolumn{6}{l}{\textit{Full-space, flat geometry}} \\
    Arithmetic Mean (AM)~\cite{Wortsman_2022_CVPR}   & $\mI_{m}$  & $1$ & $\vzero$ & $\mI$ & Uniform weights, flat geometry. \\
    Task Arithmetic (TA)~\cite{ilharco2023editingmodelstaskarithmetic}   & $\mI_{m}$  & $\lambda_{t}$ & $\vtheta_{0}$ & $\mI$ & Flat local geometry.\\
    NAN~\cite{si2025nantrainingfreesolutioncoefficient} & $\mI_{m}$ & $\left\|\vtheta_{t}\right\|_{F}^{-1}$ & $\vzero$ & $\mI$ & Flat local geometry.\\
    TIES-Merging~\cite{NEURIPS2023_1644c9af} & $\mI_{m}$ & $\lambda_{t}$ & $\vtheta_{0}$ & Binary mask & Biased geometry defined by rules. \\
    \midrule
    \multicolumn{6}{l}{\textit{Full-space, curvature-aware}} \\
    Fisher Averaging (FA)~\cite{NEURIPS2022_70c26937} & $\mI_{m}$ & $\lambda_{t}$ & $\vzero$ & $\text{diag}(\mF_{t})$ & Diagonal approximation. \\
    RegMean~\cite{jin2025datalessknowledgefusionmerging} & $\mI_{m}$ & $1$ & $\vzero$ & $\mX_{t}^{\intercal}\mX_{t}$ & Requires task data. \\
    Daheim~\textit{et al.}~\cite{daheim2024modelmerginguncertaintybasedgradient} & $\mI_{m}$ & $\lambda_{t}$ & $\vtheta_{0}$ & $\mH_{t}$ (approx.) & Full-space Hessian intractable. \\
    \midrule
    \multicolumn{6}{l}{\textit{Subspace, flat geometry}} \\
    TSV-M~\cite{gargiulo2025tasksingularvectorsreducing} & Top-$r$ SVD & $\lambda_{t}$ & $\vtheta_{0}$ & $\mI_{p}$ & Ignores curvature. \\
    Isotropic Merging~\cite{marczak2025taskleftbehindisotropic} & Isotropy-adj. & $\lambda_{t}$ & $\vtheta_{0}$ & $\mI_{p}$ & Ignores curvature. \\
    \midrule
    \textbf{EpiMer (Ours)} & Top-$r$ SVD & $\lambda_{t}$ & $\vtheta_{0}$ & $\mS^{\intercal}\mH_{t}\mS$ & --- \\
    \bottomrule
    \end{tabular}
    \label{tab: relation}
\end{table*}

\subsection{Subspace Solver}
\label{subsec: subspace-epimer}

To begin with, we approximate each task's geodesic to the merged model by a linear path $\gamma(t)\gets{(1-t)\vtheta_{t}+t\vtheta_{m}}$, which by~\cref{asm: connected-basin} stays inside a low-change-of-loss basin. With $\boldsymbol{\delta}_{t}\triangleq\vtheta_{t}-\vtheta_{0}$ and a metric pinned to its endpoint $\tG_{t}\gets\mH_{t}\triangleq\E_{\vx\sim\gD_{t}}\mH_{t}(\vtheta_{t})$, the Fr\'echet objective in~\eqref{eq: frechet-objective} reduces to a quadratic in $\boldsymbol{\delta}_{m}\triangleq\vtheta_{m}-\vtheta_{0}$,
\begin{equation}
    \boldsymbol{\delta}_{m}^{*}=\underset{\boldsymbol{\delta}_{m}}{\arg\min}\sum_{t=1}^{T}\lambda_{t}(\boldsymbol{\delta}_{m}-\boldsymbol{\delta}_{t})^{\top}\mH_{t}(\boldsymbol{\delta}_{m}-\boldsymbol{\delta}_{t}),
    \label{eq: affine-obj}
\end{equation}
which has a closed-form minimizer
\begin{equation}
    \boldsymbol{\delta}_{m}^{*}=\Bigl(\sum_{t=1}^{T}\lambda_{t}\mH_{t}\Bigr)^{-1}\Bigl[\sum_{t=1}^{T}\lambda_{t}\mH_{t}\boldsymbol{\delta}_{t}\Bigr],
    \label{eq: affine-sol}
\end{equation}
yielding $\vtheta_{m}=\vtheta_{0}+\boldsymbol{\delta}_{m}^{*}$ (derivation in~\cref{appx: derivations}). We propose to approximate each $\mH_{t}$ by the empirical Fisher diagonal at the converged parameters,
\begin{equation}
    \mH_{t}\approx\text{diag}(\vv_{t}),\;\vv_{t}=\E_{\vx\sim\gD_{t}}\left[(\nabla_{\vtheta}\gL_{t}(\vx,\vtheta_{t}))^{2}\right],
    \label{eq: hessian-approx}
\end{equation}
which is computable in a single forward-backward pass. We later demonstrate in~\cref{subsec: data-efficiency} that this remains accurate on as little as $0.5\%$ of each task's training set.

\subsubsection{Merge in Subspace}
\label{subsubsec: subspace-id}

However, directly solving~\eqref{eq: affine-sol} in the $m$-dimensional ambient space is problematic. On the one hand, approximating the true geodesic by a linear path can be invalid in the high-dimensional degenerated Riemannian space. On the other hand, materializing $\sum_{t}\lambda_{t}\mH_{t}$ comes with quadratic cost in $m$, and the diagonal proxy in~\eqref{eq: hessian-approx} shares the coordinate axes as its eigenbasis for~\emph{every} task, so the per-task curvatures can differ only in axis-aligned scales and hence the matrix-weighted solve degenerates to a re-weighting of Task Arithmetic~\cite{ilharco2023editingmodelstaskarithmetic}. We propose to resolve both pathologies by restricting the merge to a low-rank \emph{subspace matrix} $\mS\in\mathbb{R}^{m\times p}$ with $p\ll m$ and column-orthonormal columns ($\mS^{\top}\mS=\mI_{p}$), whose columns form a basis for a $p$-dimensional linear subspace of the original parameter space. In this subspace, the projected Hessian $\tilde{\mH}_{t}=\mS^{\top}\mH_{t}\mS\in\mathbb{R}^{p\times p}$ becomes \emph{dense} even when $\mH_{t}$ is diagonal, recovering the off-diagonal curvature signal that the full-space proxy disregards.

To achieve this, our primary challenge is to construct the columns of $\mS$ so that they span the directions \emph{all tasks} actually use. A shared subspace formed by joint orthogonalization of per-task singular factors has a hard rank ceiling at $T$, which can collapse EpiMer onto Task Arithmetic. We instead adopt the \emph{per-task factored basis} introduced in TSV-M~\cite{gargiulo2025tasksingularvectorsreducing}. For each neural network layer $\ell=1,\ldots,L$, we apply the Singular Value Decomposition (SVD) to each task vector and keep the top-$k$ triples, concatenate the truncated left and right factors across tasks, and apply orthogonal Procrustes whitening to obtain column-orthonormal $\mU_{\perp}^{(\ell)}, \mV_{\perp}^{(\ell)}\in\mathbb{R}^{d_{\ell}\times{kT}}$. The per-layer basis is then a set of $kT$ rank-1 outer products
\begin{equation}
    \gB^{(\ell)} = \left\{\mU_{\perp,i}^{(\ell)} \mV_{\perp,i}^{(\ell)\top}\right\}_{i=1}^{kT},
    \label{eq: subspace-basis}
\end{equation}
where each atom is \emph{tagged} by the task whose SVD produced it and the per-layer dimension $kT$ removes the rank ceiling. The layer-$\ell$ block of the subspace matrix, $\mS^{(\ell)}\in\mathbb{R}^{d_{\ell}^{2}\times kT}$, collects the vectorized atoms $\{\mathrm{vec}(\mU_{\perp,i}^{(\ell)}\mV_{\perp,i}^{(\ell)\top})\}_{i=1}^{kT}$ as its columns, and the full subspace matrix $\mS\in\mathbb{R}^{m\times p}$ with $p=LkT$ is the block-diagonal concatenation of these per-layer blocks across all $L$ layers. Because each atom is orthonormal in the Frobenius inner product, $\mS$ has orthonormal columns by construction. See~\cref{appx: derivations} for the detailed derivation.

\subsubsection{Fr\'echet Mean in Subspace}
\label{subsubsec: subspace-sol}

With the subspace identified, we can evaluate the projected per-layer Hessian in the subspace by
\begin{equation}
    \tilde{\mH}_t^{(\ell)} = \mS^{(\ell)\top} \mH_t^{(\ell)} \mS^{(\ell)} \in \mathbb{R}^{kT \times kT},
    \label{eq: projected-hessian}
\end{equation}
which needs only $kT$ Hessian-vector products (HVPs) and is small enough to invert exactly. Crucially, even when $\mH_{t}$ is approximated by the empirical Fisher diagonal $\vv_{t}$, the projection
\begin{equation}
    \tilde{\mH}_t^{(\ell)} = \mS^{(\ell)\top} \, \text{diag}(\vv_t^{(\ell)}) \, \mS^{(\ell)}
    \label{eq: projected-hessian-fisher}
\end{equation}
yields a \emph{dense} $kT\times kT$ matrix that captures the cross-parameter curvature lost by the full-space diagonal proxy. Writing $\tilde{\boldsymbol{\delta}}_{t}^{(\ell)}=\mS^{(\ell)\top}\boldsymbol{\delta}_{t}^{(\ell)}$ for the projected per-task displacement, the per-layer subspace Fr\'echet mean is
\begin{equation}
    \tilde{\boldsymbol{\delta}}_{m}^{(\ell)} = \left(\sum_{t=1}^{T} \lambda_t \tilde{\mH}_t^{(\ell)}\right)^{-1} \left[\sum_{t=1}^{T} \lambda_t \tilde{\mH}_t^{(\ell)} \tilde{\boldsymbol{\delta}}_t^{(\ell)}\right],
    \label{eq: subspace-frechet-mean}
\end{equation}
and the merged model is recovered by lifting back through $\mS$, i.e.\ $\vtheta_m = \vtheta_0 + \mS\,\tilde{\boldsymbol{\delta}}_m^{*}$.

\Cref{alg: subspace-epimer} presents the full algorithm. Moreover, our~\cref{thm: existence} describes the existence and uniqueness of such solution in the subspace. See proof in~\cref{appx-subsec: subspace-proofs}.

\begin{thm}[Existence and Uniqueness]
\label[theorem]{thm: existence}
If we have the Hessian $\bar{\mH} = \sum_{t=1}^T \lambda_t \tilde{\mH}_t \succ 0$, then $\gF(\tilde{\boldsymbol{\delta}}) = \sum_{t} \lambda_t (\tilde{\boldsymbol{\delta}} - \tilde{\boldsymbol{\delta}}_t)^{\top} \tilde{\mH}_t (\tilde{\boldsymbol{\delta}} - \tilde{\boldsymbol{\delta}}_t)$ admits a unique minimizer in~\eqref{eq: subspace-frechet-mean}.
\end{thm}

\begin{algorithm}[!t]
    \caption{Epistemic Merging}
    \label{alg: subspace-epimer}
    \begin{algorithmic}[1]
        \Require Pre-trained parameters $\vtheta_{0}$,
        \Require fine-tuned parameters $\{\vtheta_{t}\mid{t=1,\ldots,T}\}$
        \Require Per-task rank $k$, global rescaling $\alpha$
        \Ensure Merged parameters $\vtheta_{m}$
        \Statex
        \Statex \(\triangleright\) Build the per-task subspace basis.
        \For{$t=1$ to $T$}
            \State $\bm{\delta}_{t}\gets\vtheta_{t}-\vtheta_{0}$ \Comment{task vector}
            \For{layer $\ell=1$ to $L$}
                \State $\mU_{t}^{(\ell)}\bm{\Sigma}_{t}^{(\ell)}\mV_{t}^{(\ell)\top}\gets\mathrm{SVD}\bigl(\bm{\delta}_{t}^{(\ell)}\bigr)$
                \State keep the top-$k$ triples $\bigl(\tilde{\mU}_{t}^{(\ell)},\tilde{\bm{\Sigma}}_{t}^{(\ell)},\tilde{\mV}_{t}^{(\ell)}\bigr)$
            \EndFor
        \EndFor
        \For{layer $\ell=1$ to $L$}
            \State stack $\mU^{(\ell)}\gets\bigl[\tilde{\mU}_{1}^{(\ell)}\,|\,\cdots\,|\,\tilde{\mU}_{T}^{(\ell)}\bigr]$ and for $\mV^{(\ell)}$
            \State $\mU_{\perp}^{(\ell)},\mV_{\perp}^{(\ell)}\gets\mathrm{OrthogonalProcrustes}\bigl(\mU^{(\ell)},\mV^{(\ell)}\bigr)$
            \State assemble per-layer basis $\mS^{(\ell)}$ whose columns are $$\mathrm{vec}\bigl(\mU_{\perp,i}^{(\ell)}\mV_{\perp,i}^{(\ell)\top}\bigr),\ i=1,\ldots,kT$$
        \EndFor
        \Statex
        \Statex \(\triangleright\) Solve the subspace Fr\'echet mean per layer.
        \For{layer $\ell=1$ to $L$}
            \For{task $t=1$ to $T$}
            \State project task vectors $\tilde{\bm{\delta}}_{t}^{(\ell)}\gets{\mS^{(\ell)}}^{\top}\bm{\delta}_{t}^{(\ell)}$
            \State project Hessians $\tilde{\mH}_{t}^{(\ell)}\gets{\mS^{(\ell)}}^{\top}\mH_{t}^{(\ell)}\mS^{(\ell)}$
            \EndFor
            \State solve $\tilde{\bm{\delta}}_{m}^{(\ell)}\gets\alpha\,\bar{\mH}^{-1}\sum_{t=1}^{T}\tilde{\mH}_{t}^{(\ell)}\tilde{\bm{\delta}}_{t}^{(\ell)}$ (\eqref{eq: frechet-sum})
        \EndFor
        \Statex
        \Statex \(\triangleright\) Lift back to the ambient parameter space.
        \State $\vtheta_{m}\gets\vtheta_{0}+\mS\,\tilde{\bm{\delta}}_{m}^{*}$ \Comment{block-diagonal lift across layers}
    \end{algorithmic}
\end{algorithm}

\subsubsection{Aggregation on Tagged Basis}
\label{subsubsec: aggregator}

On the tagged basis, \eqref{eq: subspace-frechet-mean} can suffer from \emph{magnitude collapse}. Specifically, each atom is owned by exactly one task, but the Fr\'echet \emph{average} divides each block by $T$ which severely under-merges the models. We propose a easy fix to interpret the per-task contributions as a \emph{sum} while preserving the curvature reweighting, which is given by
\begin{equation}
    \tilde{\boldsymbol{\delta}}_m^{(\ell)} = \alpha\,\bar{\mH}^{-1}\,\sum_{t=1}^T \tilde{\mH}_t^{(\ell)} \tilde{\boldsymbol{\delta}}_t^{(\ell)},
    \quad \bar{\mH} = \frac{1}{T}\sum_{t=1}^T \tilde{\mH}_t^{(\ell)},
    \label{eq: frechet-sum}
\end{equation}
where $\alpha$ is a global rescaling shared with TSV-M. We note that \eqref{eq: frechet-sum} equals $\alpha T$ times \eqref{eq: subspace-frechet-mean} by undoing the averaging. In a homogeneous-curvature limit, it collapses to $\alpha\sum_{t}\tilde{\bm{\delta}}_{t}$, recovering the method in TSV-M~\cite{gargiulo2025tasksingularvectorsreducing}. Otherwise, the matrix solve reshapes the contributions of each task corresponding to the curvature. Our $\alpha$-sweep experiment in~\cref{subsec: alpha-sweep-main} confirms that~\eqref{eq: frechet-sum} dominates~\eqref{eq: subspace-frechet-mean} at every rank and backbone.

\subsection{Theoretical Analysis}
\label{subsec: theory}

We present three theoretical results that characterize the behavior of subspace merging. See proofs in~\cref{appx: derivations}.

\begin{thm}[Merging Error Bound]
\label{thm: error-bound}
Under Assumptions~\ref{asm: smooth}--\ref{asm: local-opt}, the merged model $\vtheta_m = \vtheta_0 + \mS\tilde{\boldsymbol{\delta}}_m^*$ satisfies:
\begin{equation}
    \sum_{t=1}^{T} \lambda_t \left[\gL_t(\vtheta_m) - \gL_t(\vtheta_t)\right] \leq \frac{1}{2}\left(\sqrt{\gV_S} + \sqrt{\gR_S}\right)^2 + R_3,
    \label{eq: error-bound}
\end{equation}
where $\gV_S = \sum_t \lambda_t (\tilde{\boldsymbol{\delta}}_m^* - \tilde{\boldsymbol{\delta}}_t)^{\top} \tilde{\mH}_t (\tilde{\boldsymbol{\delta}}_m^* - \tilde{\boldsymbol{\delta}}_t)$ is the \textbf{subspace Fr\'echet variance} reflecting the irreducible conflict between tasks, $\gR_S = \sum_t \lambda_t (\boldsymbol{\delta}_t^{\perp})^{\top} \mH_t \boldsymbol{\delta}_t^{\perp}$ with $\boldsymbol{\delta}_t^{\perp}=(\mI_m - \mS\mS^{\top})\boldsymbol{\delta}_t$ is the \textbf{residual energy} as information lost by projection, and $R_3 = \gO(\max_t \|\vtheta_m - \vtheta_t\|^3)$ is the Taylor remainder.
\end{thm}

This decomposition helps to understand \emph{why} existing spectral methods work. TSV-M~\cite{gargiulo2025tasksingularvectorsreducing} proposes to select $\mS$ to minimize $\sum_t \|\boldsymbol{\delta}_t^{\perp}\|^2$, which corresponds to $\gR_S$ with $\mH_t = \mI$, while ignoring $\gV_S$ entirely. Our method addresses the limitation by minimizing $\gV_S$ on the per-task basis, and the spectral construction in~\eqref{eq: subspace-basis} helps to keep $\gR_S$ small at the rank required by the curvature solve.

\begin{thm}[Curvature Advantage]
\label{thm: curvature-advantage}
Let $\tilde{\boldsymbol{\delta}}_{I} = \bar{\boldsymbol{\delta}} = \sum_t \lambda_t \tilde{\boldsymbol{\delta}}_t$ denote the flat-geometry solution and $\tilde{\boldsymbol{\delta}}_{H} = \bar{\mH}^{-1}\sum_t \lambda_t \tilde{\mH}_t \tilde{\boldsymbol{\delta}}_t$ be the curvature-aware Fr\'echet mean. The improvement in the merging objective is:
\begin{equation}
    \gF(\tilde{\boldsymbol{\delta}}_I) - \gF(\tilde{\boldsymbol{\delta}}_H) = \vc^{\top} \bar{\mH}^{-1} \vc \geq 0,
    \label{eq: curvature-advantage}
\end{equation}
where $\vc = \sum_t \lambda_t (\tilde{\mH}_t - \bar{\mH})(\tilde{\boldsymbol{\delta}}_t - \bar{\boldsymbol{\delta}})$ is the factor that reflects a \textbf{curvature-task correlation}.
\end{thm}

Herein, we notice that the quantity $\eta = \vc^{\top}\bar{\mH}^{-1}\vc$ is computable in $\gO(p^3)$ from the projected Hessians and task vectors. It equals zero when (a)~all tasks share identical projected Hessians (\emph{i.e., homogeneous curvature}), (b)~all task vectors are identical, or (c)~curvature deviations and task-vector deviations are uncorrelated. In these cases, flat-geometry merging is near-optimal. When $\eta \gg 0$, curvature-aware merging is justified. We highlight that this diagnostic is available \emph{only} through the scope of Riemannian framework, while existing methods fail to assess whether their flat-geometry assumption is adequate. Notably, through the lens of geometry, we can unify existing methods as special cases with different geometric metrics in~\cref{prop: subsumption}.

\begin{prop}[Subsumptions]
\label{prop: subsumption}
The subspace Fr\'echet mean in~\eqref{eq: subspace-frechet-mean} recovers existing methods as special cases:
\begin{enumerate}[leftmargin=1.7em]
    \item $\mS = \mI_m$, $\tilde{\mH}_t = \mI_m$: Task Arithmetic~\cite{ilharco2023editingmodelstaskarithmetic}.
    \item $\mS = \mI_m$, $\tilde{\mH}_t = \text{diag}(\mF_t)$: Fisher Averaging~\cite{NEURIPS2022_70c26937}.
    \item $\mS = \mI_m$, $\tilde{\mH}_t = \mH_t$: Gradient Matching~\cite{daheim2024modelmerginguncertaintybasedgradient}.
    \item $\mS = $ top-$k$ SVD, $\tilde{\mH}_t = \mI_p$: TSV-M~\cite{gargiulo2025tasksingularvectorsreducing}.
    \item $\gB = $ per-task tagged rank-1 basis (\eqref{eq: subspace-basis}), $\tilde{\mH}_t = \mS^{\top}\mH_t\mS$, aggregator from~\eqref{eq: frechet-sum}: \textbf{EpiMer}.
\end{enumerate}
This establishes an unification spanning both the curvature-aware tradition (cases 2--3) and the spectral tradition (case 4) under a single equation. See~\cref{tab: relation} for more details.
\end{prop}

\section{Experiment}
\label{sec: experiment}

We report merging fine-tuned CLIP-ViT models on eight image classification tasks. We find them to reside in low-change-of-loss basins (see PCA loss landscapes in~\cref{appx: loss-landscape}). \Cref{subsec: performance} reports the average top-$1$ accuracy. \Cref{subsec: alpha-sweep-main,subsec: diagnostic-eta,subsec: data-efficiency,subsec: worst-task} respectively study the global rescaling sensitivity with respect to $\alpha$, the curvature heterogeneity diagnostic proposed in~\cref{thm: curvature-advantage}, the data-efficiency of the empirical Fisher diagonal, and worst-task robustness. We present further implementation details and additional results in~\cref{appx: implementation,appx: results}.

\subsection{Experiment Setup}
\label{subsec: setup}

\noindent\textbf{Models and Tasks.} Following prior works~\cite{Wortsman_2022_CVPR,pmlr-v162-wortsman22a,ilharco2023editingmodelstaskarithmetic}, we evaluate model merging on eight image classification tasks: Stanford Cars~\cite{KrauseStarkDengFei-Fei_3DRR2013}, DTD~\cite{cimpoi14describing}, EuroSAT~\cite{helber2019eurosatnoveldatasetdeep}, GTSRB~\cite{6033395}, MNIST~\cite{lecun2010mnist}, RESISC45~\cite{Cheng_2017}, SUN397~\cite{sun397}, and SVHN~\cite{netzer2011reading}. To facilitate a fair comparison, we use publicly released CLIP-ViT checkpoints\footnote{\url{https://github.com/mlfoundations/task_vectors}}. See implementation details in~\cref{appx: implementation}.

\noindent\textbf{Benchmark.} We compare EpiMer against existing \textit{flat-geometry methods}, including Arithmetic Mean (AM)~\cite{Wortsman_2022_CVPR}, Task Arithmetic~\cite{ilharco2023editingmodelstaskarithmetic}, TIES-Merging~\citep{NEURIPS2023_1644c9af}, and TSV-M~\cite{gargiulo2025tasksingularvectorsreducing}, as well as the \textit{curvature-aware} Fisher-weighted Averaging~\citep{NEURIPS2022_70c26937}. Linear probing and per-task fine-tuning anchor the table as reference baselines. We report average top-$1$ accuracy across the eight tasks as the main metric in~\cref{tab: performance}.

\noindent\textbf{Empirical Fisher diagonal.} The publicly released CLIP-ViT checkpoints ship only the converged parameters $\vtheta_{t}$, not the optimizer states. We therefore need to reconstruct the empirical Fisher diagonal $\vv_{t}$ used by~\eqref{eq: hessian-approx} and~\eqref{eq: projected-hessian-fisher} by a single forward-backward pass over each task's training set at $\vtheta_{t}$, accumulating $(\nabla_{\vtheta}\gL_t(\vx,\vtheta_t))^{2}$ on the fly. The same procedure supplies Fisher-weighted Averaging, so both curvature-aware methods consume identical curvature inputs. \Cref{subsec: data-efficiency} further validates that this estimate is robust to a tiny fraction of the training data.

\begin{table}[!t]
\centering
\caption{Average top-$1$ accuracy across the eight image classification tasks for each merged CLIP-ViT variant. Best merging method per column is \textbf{bold}, second best is \underline{underlined}. All baselines use $\lambda_t = 1/T$, so AM and TA collapse to the same merged delta and are listed jointly. TSV-M and EpiMer are reported at $k=32$ and the per-method optimal global rescaling $\alpha$ from~\cref{subsec: alpha-sweep-main}: TSV-M uses $\alpha=0.70$ on ViT-B/32 and $\alpha=1.0$ on ViT-B/16 and ViT-L/14; EpiMer uses $\alpha=1.0$ on all three backbones. Per-task breakdowns on every dataset are deferred to~\cref{tab: performance-full}. On ViT-L/14, EpiMer exceeds TSV-M by $0.06$ percentage points ($0.9065$ vs $0.9059$); both display as $.906$ at the three-decimal precision of this table.}
\label{tab: performance}
\begin{tabular}{@{}lccc@{}}
\toprule
\textbf{Method} & \textbf{ViT-B/32} & \textbf{ViT-B/16} & \textbf{ViT-L/14} \\
\midrule
\multicolumn{4}{l}{\textit{Baselines}} \\
Linear-Probing & .480 & .553 & .649 \\
Fine-tuning    & .909 & .929 & .943 \\
\midrule
\multicolumn{4}{l}{\textit{Flat-geometry methods}} \\
AM~\cite{Wortsman_2022_CVPR} / TA~\cite{ilharco2023editingmodelstaskarithmetic}  & .653 & .710 & .791 \\
TIES~\cite{NEURIPS2023_1644c9af}                                                 & .725 & .774 & .859 \\
TSV-M~\cite{gargiulo2025tasksingularvectorsreducing}                             & \underline{.822} & \underline{.865} & \underline{.906} \\
\midrule
\multicolumn{4}{l}{\textit{Curvature-aware methods}} \\
Fisher~\cite{NEURIPS2022_70c26937} & .539 & .625 & .720 \\
\textbf{EpiMer (Ours)}    & \textbf{.833} & \textbf{.870} & \textbf{.906} \\
\bottomrule
\end{tabular}
\end{table}

\begin{table*}[!t]
    \centering
    \caption{Ablation results on global rescaling $\alpha$ sensitivity at $k=32$ on all three backbones. Best per row is \textbf{bold}. We report more results with rank $r\in\{4,16\}$ in~\cref{appx: alpha-sweep}.}
    \begin{tabular}{@{}llcccccc@{}}
    \toprule
    \textbf{Backbone} & \textbf{Method} & $\alpha{=}0.20$ & $\alpha{=}0.30$ & $\alpha{=}0.40$ & $\alpha{=}0.50$ & $\alpha{=}0.70$ & $\alpha{=}1.00$ \\
    \midrule
    \multirow{2}{*}{ViT-B/32} & TSV-M             & .630 & .699 & .750 & .787 & \textbf{.822} & .820 \\
                              & EpiMer   & .601 & .670 & .724 & .764 & .812 & \textbf{.833} \\
    \midrule
    \multirow{2}{*}{ViT-B/16} & TSV-M             & .688 & .747 & .792 & .822 & .857 & \textbf{.865} \\
                              & EpiMer   & .666 & .724 & .771 & .804 & .846 & \textbf{.870} \\
    \midrule
    \multirow{2}{*}{ViT-L/14} & TSV-M             & .772 & .816 & .849 & .870 & .895 & \textbf{.906} \\
                              & EpiMer   & .766 & .808 & .841 & .863 & .890 & \textbf{.906} \\
    \bottomrule
    \end{tabular}
    \label{tab: alpha-sweep-r32}
\end{table*}

\subsection{Merged Model Performance}
\label{subsec: performance}

\Cref{tab: performance} presents the across-task average top-$1$ accuracy of each merged model on the three backbones. Results show that the proposed EpiMer is the best merging method on every backbone, beating TSV-M by $1.10\%$, $0.48\%$, and $0.06\%$ on ViT-B/32, ViT-B/16, and ViT-L/14 (at rank $k=32$ and the per-method optimal $\alpha$ from~\cref{subsec: alpha-sweep-main}), and TIES-Merging by $10.8\%$, $9.6\%$, and $4.7\%$, respectively. The Fisher Averaging baseline collapses on every backbone, confirming that a diagonal Fisher in the full parameter space is \emph{too coarse a curvature estimate}, and the proposed EpiMer manages to address the issue by subspace projection. Furthermore, we highlight the key roles of subspace projection and curvature awareness in the improved performance.

\noindent\textbf{Role of Subspace.} Against the non-spectral baselines (\emph{i.e., AM/TA, TIES, Fisher Averaging}), we note that subspace projection plays an important role. It compresses each task vector into the union of its top-$k$ singular directions removes the noise floor and concentrates the merge where any task has signal. On ViT-B/32, projecting AM/TA's delta onto the same tagged basis already lifts accuracy from $0.653\%$ to TSV-M's $0.822$\%. Our curvature-aware solution then further pushes the performance to $0.833\%$. The curvature-aware aggregator is a second-order refinement that is most important when the projected per-task Hessians are indeed heterogeneous.

\noindent\textbf{Role of Curvature.} TSV-M and EpiMer construct the merged delta on the identical per-task tagged basis $\gB^{(\ell)}$ from~\eqref{eq: subspace-basis} and differ only in the aggregation. Herein, TSV-M uses an isotropic metric, whereas EpiMer applies the matrix-weighted solve in~\eqref{eq: frechet-sum} with the projected per-task Hessian. We argue that the shrinking margin on larger backbones reflects a \emph{saturation effect} rather than a failure of curvature awareness.  Specifically, TSV-M on ViT-L/14 already reaches $0.906\%$ against the $0.943\%$ fine-tuning ceiling, leaving limited improvements for any second-order refinement. \Cref{subsec: diagnostic-eta} disentangles the diagnostic prediction from saturation by sweeping the per-task rank $k$ \emph{within} each backbone and tracking $\eta$ alongside the observed gap.

\subsection{Global Rescaling Sensitivity}
\label{subsec: alpha-sweep-main}

Both TSV-M and the matrix-weighted aggregator in~\eqref{eq: frechet-sum} exploit a global rescaling $\alpha$ that controls the Frobenius norm of the merged delta. The default $\alpha = 1/\sqrt{T}$ recommended by~\cite{gargiulo2025tasksingularvectorsreducing} is the natural choice when the per-task tagged contributions are nearly orthogonal. Nevertheless, we find it severely under-tuned on the eight-task CLIP-ViT setup of~\cref{tab: performance}. At a rank $k=32$, the optimum lies in $\alpha\in[0.7,1.0]$ for both methods, which is around $9\%$--$13\%$ above the $1/\sqrt{T}=1/\sqrt{8}\approx 0.354$ default. To make the comparison in~\cref{tab: performance} as fair as possible, we only report numbers with the per-method optimal $\alpha$. In addition, we conduct an ablation study and~\cref{tab: alpha-sweep-r32} reports the sweep at $k=32$ on every backbone. We also report the matching $k\in\{4,16\}$ results in~\cref{appx: alpha-sweep}.

Two observations emerge from the results. First, the per-method optimum lies in $\alpha\in[0.7,1.0]$ on both methods, far from the literature default $1/\sqrt{T}\approx 0.354$. Thereby, tuning $\alpha$ alone closes most of the gap to fine-tuning. Meanwhile, at the optimum, the proposed EpiMer dominates TSV-M on every backbone and at every rank in~\cref{appx: alpha-sweep}. The matrix-weighted aggregator never falls behind, even in the saturation regime on ViT-L/14, where TSV-M is already within $3.7\%$ of the per-task fine-tuning ceiling.

\subsection{Curvature Heterogeneity Diagnostic}
\label{subsec: diagnostic-eta}

Our~\cref{thm: curvature-advantage} predicts that the improvement of curvature-aware merging over flat-geometry merging is governed by the curvature heterogeneity $\eta = \vc^{\intercal}\bar{\mH}^{-1}\vc$. We investigate this diagnostic on a per-task, tagged basis, as used by both TSV-M and EpiMer. At each rank $k$, we project the per-task task vectors and Hessians into the basis of~\eqref{eq: subspace-basis}, normalize each Hessian by trace, and compute $\eta$. \Cref{fig: rank-vs-eta-and-margin} reports $\eta$ alongside the observed accuracy margin (EpiMer $-$ TSV-M, both at their per-method optimal $\alpha$) for $k\in\{2,4,8,16,32\}$ on all three backbones.

Within each backbone, $\eta$ rises monotonically with $k$ as the basis admits more curvature heterogeneity, and EpiMer maintains a positive margin over TSV-M at every $k$ across all backbones. However, the cross-backbone ordering of $\eta$ \emph{does not} predict the cross-backbone ordering of the improvement margin. Specifically, ViT-L/14 has the largest $\eta$ at every rank yet the smallest margin. We argue that this is because its TSV-M baseline already sits within $3.7\%$ of the per-task fine-tuning ceiling. The diagnostic therefore behaves as a within-backbone signal of how much theoretical room the curvature-aware aggregator has to operate, modulated by saturation and by the global $\alpha$ tuning.

\begin{figure}[!t]
    \centering
    \includegraphics[width=\linewidth]{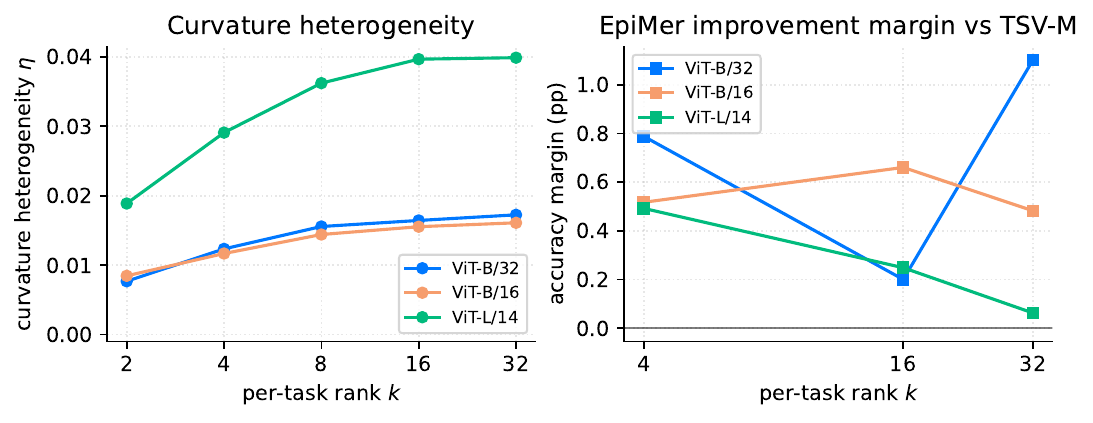}
    \caption{Curvature heterogeneity diagnostic at varying per-task ranks $k$. \textbf{Left:} $\eta = \vc^{\intercal}\bar{\mH}^{-1}\vc$ on the tagged basis of~\eqref{eq: subspace-basis}. \textbf{Right:} per-task mean accuracy margin in percentage at the per-method optimal $\alpha$.
    Within each backbone, $\eta$ rises monotonically with respect to the rank $k$.
    }
    \label{fig: rank-vs-eta-and-margin}
\end{figure}

\subsection{Empirical Fisher Robustness}
\label{subsec: data-efficiency}

A natural concern in the proposed method is whether the per-task empirical Fisher diagonal $\vv_{t}$ in~\eqref{eq: hessian-approx} and~\eqref{eq: projected-hessian-fisher} requires the full training set. We sweep the subsample fraction $f\in\{0.5\%,1\%,5\%,10\%,25\%,50\%,100\%\}$ and plug the resulting $\{\vv_{t}^{(f)}\}$ into EpiMer at $k=32, \alpha=1.0$ on every backbone. See detailed setup in~\cref{appx: small-data}.

\begin{figure}[!t]
    \centering
    \includegraphics[width=\linewidth]{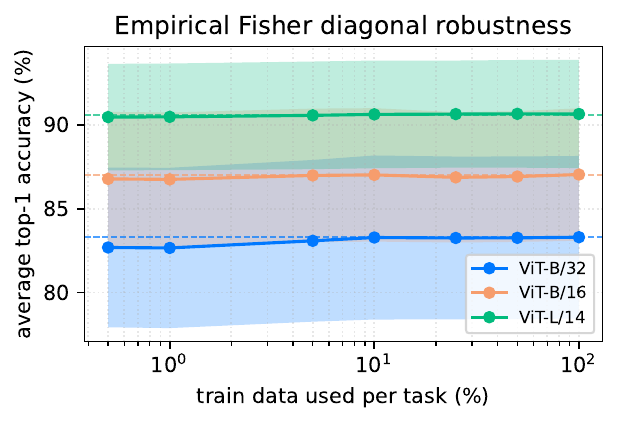}
    \caption{Average top-$1$ accuracy across the eight image classification tasks against different subsample fraction used to construct the per-task empirical Fisher diagonal $\vv_{t}$. Shaded bands show mean $\pm$ one standard error of the mean across the eight tasks; dashed lines mark the full-data accuracy from~\cref{tab: performance}.
    }
    \vspace{-2em}
    \label{fig: fisher subsample}
\end{figure}

\Cref{fig: fisher subsample} reports the sweep. The merged-model accuracy is essentially flat: even at $f=0.5\%$ (one to six batches per task at batch size $64$), accuracy lands within $\sim 0.7\%$ of the full-data value, and by $f=10\%$ the curve has saturated. The diagonal of $\E_{\vx}[(\nabla\gL)^{2}]$ converges quickly at the converged parameters, so we consider the curvature-aware solve in~\eqref{eq: frechet-sum} \emph{robust} to its only data-dependent input, with data requirements orders of magnitude smaller than test-time-adaptation baselines~\cite{yang2024adamergingadaptivemodelmerging}.

\subsection{Worst-Task Robustness}
\label{subsec: worst-task}

Another interesting question will be: \emph{does the curvature-aware gain come at the cost of already-weak tasks?} \Cref{tab: worst-task} reports the minimum per-task top-$1$ accuracy, extracted from~\cref{tab: performance-full}.
EpiMer lifts worst-task accuracy by $0.8\%$, $2.4\%$, and $0.1\%$ over TSV-M, and by $13.3\%$, $+15.1\%$, and $8.8\%$ over TIES. Therefore, curvature awareness helps tightens the worst-task margin on every backbone while keeping the average strictly ahead.

\begin{table}[!t]
    \centering
    \caption{Worst-task top-$1$ accuracy on the eight-task CLIP-ViT benchmark, i.e.~the minimum over the eight per-task cells in~\cref{tab: performance-full}. Best merging method per column is \textbf{bold}, second best is \underline{underlined}. EpiMer has the highest worst-task accuracy among all merging methods on every backbone.}
    \label{tab: worst-task}
    \resizebox{.8\linewidth}{!}{%
    \begin{tabular}{@{}lccc@{}}
    \toprule
    \textbf{Method} & \textbf{ViT-B/32} & \textbf{ViT-B/16} & \textbf{ViT-L/14} \\
    \midrule
    AM~\cite{Wortsman_2022_CVPR} / TA~\cite{ilharco2023editingmodelstaskarithmetic}    & .506 & .504 & .619 \\
    TIES~\cite{NEURIPS2023_1644c9af}                                                    & .532 & .559 & .675 \\
    Fisher~\cite{NEURIPS2022_70c26937}                                                  & .338 & .408 & .507 \\
    TSV-M~\cite{gargiulo2025tasksingularvectorsreducing}                                & \underline{.657} & \underline{.686} & \underline{.762} \\
    \textbf{EpiMer (Ours)}                                                      & \textbf{.665} & \textbf{.710} & \textbf{.763} \\
    \bottomrule
    \end{tabular}
    }
    \vspace{-1em}
\end{table}

\section{Related Works}
\label{sec: related}

\noindent\textbf{Delta-parameter merging.} Delta-parameter methods edit the residual from a shared pre-trained checkpoint. Task Arithmetic~\cite{ilharco2023editingmodelstaskarithmetic} scales and sums per-task deltas with fixed coefficients. TIES-Merging~\cite{NEURIPS2023_1644c9af} resolves cross-task interference by pruning redundant entries and aligning conflicting signs, while DARE~\cite{yu2024languagemodelssupermario,deng2025dareextremerevisitingdeltaparameter} randomly drops delta entries and rescales the survivors.

\noindent\textbf{Spectral and subspace methods.} Spectral methods project the merged delta onto a shared low-rank subspace of the task vectors. TSV-M~\cite{gargiulo2025tasksingularvectorsreducing} takes the truncated SVD of each task vector and concatenates the top singular directions into a joint per-task basis, and Isotropic Merging~\cite{marczak2025taskleftbehindisotropic} decomposes the merged delta into task-common and task-specific components, each equalized to isotropic covariance.

\noindent\textbf{Loss landscape and epistemic uncertainty.} Loss-landscape analyses motivate curvature-aware merging. Mode connectivity~\cite{vrabel2024inputspacemodeconnectivity,garipov2018losssurfacesmodeconnectivity} shows that independently trained minima can be linked by low-loss paths that connect task-specific basins. Fisher-Weighted Averaging~\cite{NEURIPS2022_70c26937} uses the Fisher information matrix as a local uncertainty estimator to weight task parameters, which~\cite{daheim2024modelmerginguncertaintybasedgradient} later generalizes through a gradient-mismatch objective.

\section{Conclusion}
\label{sec: conclusion}

In this paper, we propose EpiMer, a framework that recasts model merging as the Fr\'echet mean on a Riemannian manifold restricted to the task-vector subspace, where the projected Hessian can be computed exactly. Our analysis decomposes merging error into subspace variance and residual energy (\cref{thm: error-bound}), characterizes when curvature helps (\cref{thm: curvature-advantage}), and unifies curvature-aware and spectral methods under one solve (\cref{prop: subsumption}). The diagnostic $\eta$ captures in-subspace loss reduction but not saturation, so high-$\eta$ regimes may still show small margins. Interesting follow-up directions include generalization-aware refinements and LoRA extensions of the proposed method.

\section*{Acknowledgement}
\label{sec: acknowledgement}
We appreciate the Purdue Rosen Center for Advanced Computing for granting access to GPU clusters. We thank Brandyn White for his insightful review and suggestions.

    {
        \small
        \putbib[references/ieee_abbr,references/juanwu]
    }
\end{bibunit}

\clearpage
\appendix
\setcounter{page}{1}
\maketitlesupplementary

\begin{bibunit}
\begin{table*}[!t]
\centering
\caption{Table of notation.}
\resizebox{\textwidth}{!}{%
\begin{tabular}{@{}cl@{\hspace{3em}}cl@{}}
\toprule
    \textbf{Notation} & \textbf{Explanation} & \textbf{Notation} & \textbf{Explanation} \\
    \midrule
    $n$ & Dimension of the input data space. & $T$ & Number of fine-tuning tasks. \\
    $m$ & Dimension of the parameter space. & $L$ & Number of matrix-shaped layers in the backbone. \\
    $d_{\ell}$ & Dimension of layer $\ell$ (square layers). & $k$ & Per-task SVD truncation rank (also denoted $r$). \\
    $p$ & Subspace dimension, $p=L k T$. & $\Theta$ & Parameter space, with $\Theta\subseteq\mathbb{R}^{m}$. \\
    \addlinespace[0.35em]
    $\lambda_{t}$ & Merging weight of task $t$, with $\sum_{t}\lambda_{t}=1$. & $\alpha$ & Global Frobenius rescaling of the merged delta. \\
    $\eta$ & Curvature heterogeneity diagnostic $\vc^{\top}\bar{\mH}^{-1}\vc$. & $\tau$ & Geodesic parameter in $[0,1]$. \\
    \addlinespace[0.35em]
    $\vzero$ & All-zero vector or matrix. & $\mI_{p}$ & Identity matrix of size $p$ (likewise $\mI_{m}$). \\
    \addlinespace[0.35em]
    $\vx$ & Input data sample in $\mathbb{R}^{n}$. & $\vtheta_{0}$ & Pre-trained parameters. \\
    $\vtheta_{t}$ & Fine-tuned parameters for task $1\leq{t}\leq{T}$. & $\vtheta_{m}$ & Merged parameters. \\
    \addlinespace[0.35em]
    $\bm{\delta}_{t}$ & Task vector $\vtheta_{t}-\vtheta_{0}$ for task $t$. & $\bm{\delta}_{m}$ & Merged displacement $\vtheta_{m}-\vtheta_{0}$. \\
    $\bm{\delta}_{t}^{\perp}$ & Residual of $\bm{\delta}_{t}$ orthogonal to the subspace, $(\mI_{m}-\mS\mS^{\top})\bm{\delta}_{t}$. & $\tilde{\bm{\delta}}_{t}$ & Projected task vector $\mS^{\top}\bm{\delta}_{t}$ in the subspace. \\
    $\tilde{\bm{\delta}}_{m}$ & Merged displacement in the subspace, $\mS^{\top}\bm{\delta}_{m}$. & $\bar{\bm{\delta}}$ & Weighted average $\sum_{t}\lambda_{t}\tilde{\bm{\delta}}_{t}$. \\
    $\tilde{\bm{\delta}}_{I}$ & Flat-geometry (identity-metric) subspace solution, $\tilde{\bm{\delta}}_{I}=\bar{\bm{\delta}}$. & $\tilde{\bm{\delta}}_{H}$ & Curvature-aware subspace Fr\'echet mean, $\bar{\mH}^{-1}\!\sum_{t}\lambda_{t}\tilde{\mH}_{t}\tilde{\bm{\delta}}_{t}$. \\
    \addlinespace[0.35em]
    $\mS$ & Subspace matrix $\mS\in\mathbb{R}^{m\times p}$ with $\mS^{\top}\mS=\mI_{p}$. & $\mS^{(\ell)}$ & Layer-$\ell$ block of $\mS$ in $\mathbb{R}^{d_{\ell}^{2}\times kT}$. \\
    \addlinespace[0.35em]
    $\mU_{\perp}^{(\ell)}$ & Procrustes-whitened left factors in $\mathbb{R}^{d_{\ell}\times{kT}}$. & $\mV_{\perp}^{(\ell)}$ & Procrustes-whitened right factors in $\mathbb{R}^{d_{\ell}\times{kT}}$. \\
    $\mU_{\perp,i}^{(\ell)}$ & $i$-th column of $\mU_{\perp}^{(\ell)}$ (an atom of the tagged basis). & $\mV_{\perp,i}^{(\ell)}$ & $i$-th column of $\mV_{\perp}^{(\ell)}$ (an atom of the tagged basis). \\
    \addlinespace[0.35em]
    $\mF_{t}$ & Fisher information matrix of task $t$. & $\mH_{t}$ & Expected Hessian of task $t$ evaluated at $\vtheta_{t}$. \\
    $\tilde{\mH}_{t}$ & Projected per-task Hessian $\mS^{\top}\mH_{t}\mS\in\mathbb{R}^{p\times p}$. & $\tilde{\mH}_{t}^{(\ell)}$ & Per-layer projected Hessian $\mS^{(\ell)\top}\mH_{t}^{(\ell)}\mS^{(\ell)}\in\mathbb{R}^{kT\times kT}$. \\
    $\bar{\mH}$ & Weighted average of projected Hessians $\sum_{t=1}^{T}\lambda_{t}\tilde{\mH}_{t}$. & $\vc$ & Curvature-task correlation $\sum_{t}\lambda_{t}(\tilde{\mH}_{t}-\bar{\mH})(\tilde{\bm{\delta}}_{t}-\bar{\bm{\delta}})$. \\
    \addlinespace[0.35em]
    $\vu$ & First-moment vector in the optimizer state. & $\vv$ & Second-moment vector in the optimizer state. \\
    $\vv_{t}$ & Empirical Fisher diagonal of task $t$, $\E[(\nabla\gL_{t})^{2}]$. & & \\
    \addlinespace[0.35em]
    $\gM$ & Differentiable manifold representing the parameter space. & $\tG(\cdot)$ & Riemannian metric tensor field on $\gM$. \\
    $\gamma(\cdot)$ & Geodesic path $\gamma:[0,1]\to\gM$. & & \\
    \addlinespace[0.35em]
    $\gL_{t}(\vx,\vtheta)$ & Loss of task $t$ at input $\vx$ with parameter $\vtheta$. & $\gJ(\vtheta)$ & Multi-task objective $\sum_{t}\lambda_{t}\E[\gL_{t}]$. \\
    $\gF(\tilde{\bm{\delta}})$ & Subspace Fr\'echet objective, a quadratic in $\tilde{\bm{\delta}}$. & & \\
    \addlinespace[0.35em]
    $\gD$ & Collection of all task-specific datasets $\{\gD_{1},\ldots,\gD_{T}\}$. & $\gD_{t}$ & Task-$t$ training distribution. \\
    \addlinespace[0.35em]
    $\gB^{(\ell)}$ & Per-task tagged basis of rank-1 atoms at layer $\ell$. & $\gV_{S}$ & Subspace Fr\'echet variance (irreducible task conflict). \\
    $\gR_{S}$ & Residual energy $\sum_{t}\lambda_{t}(\bm{\delta}_{t}^{\perp})^{\top}\mH_{t}\bm{\delta}_{t}^{\perp}$ lost by projection. & & \\
    \addlinespace[0.35em]
    $\sC^{2}$ & Class of twice continuously differentiable functions. & $\mathbb{R}$ & Real numbers. \\
    \bottomrule
    \end{tabular}
}
\label{tab: notation}
\end{table*}

\section{Notations}
\label[appendix]{appx: notation}

To improve the readability for the general audience, Table~\ref{tab: notation} lists all the notation used in this paper.

\section{Additional Details for Methodology}
\label[appendix]{appx: derivations}

\subsection{Proof for Proposition 1}

The following provides a proof for proposition~\ref{prop: equiv-obj}.

\begin{proof}
    Let the global multi-task objective function be the weighted sum of the individual task losses, given by
    \begin{equation}
        \gJ(\vtheta)=\sum\limits_{t=1}^{T}\lambda_{t}\E_{\vx\sim\gD_{t}}\left[\gL_{t}(\vx, \vtheta)\right].
    \end{equation}
    Our objective is to find the optimal parameter $\vtheta_{m}$ that minimizes the objective above. To achieve this, we first consider the Taylor expansion of the expected loss for a specific task $t$ around its fine-tuned optimum $\vtheta_{t}$. Denote $L_{t}(\vtheta)=\E_{\vx\sim\gD_{t}}\left[\gL_{t}(\vx,\vtheta)\right]$ for simplicity, the second-order Taylor expansion at $\vtheta_{t}$ gives
    \begin{equation}
        \begin{aligned}
        &L_{t}(\vtheta)\approx \\
        &L_{t}(\vtheta_{t}) + (\vtheta-\vtheta_{t})^{\top}\nabla_{\vtheta}L_{t}(\vtheta_{t}) + \frac{1}{2}(\vtheta-\vtheta_{t})^{\top}\mH_{t}(\vtheta-\vtheta_{t}),
        \end{aligned}
    \end{equation}
    Based on our assumption~\ref{asm: local-opt} in the main paper, each fine-tuned model resides at a local minimum, which implies the first-order derivative vanishes, i.e., $\nabla_{\vtheta}L_{t}(\vtheta_{t})\approx\vzero$. Therefore, the expansion reduces to
    \begin{equation}
        L_{t}(\vtheta)\approx L_{t}(\vtheta_{t}) + \frac{1}{2}(\vtheta-\vtheta_{t})^{\top}\mH_{t}(\vtheta-\vtheta_{t}).
    \end{equation}
    Substituting the above approximation back into the global objective function, we have
    \begin{equation}
        \begin{aligned}
        \gJ(\vtheta)&\approx\sum\limits_{t=1}^{T}\lambda_{t}\left(L_{t}(\vtheta_{t})+\frac{1}{2}(\vtheta-\vtheta_{t})^{\top}\mH_{t}(\vtheta-\vtheta_{t})\right) \\
        &=\frac{1}{2}\sum\limits_{t=1}^{T}\lambda_{t}(\vtheta-\vtheta_{t})^{\top}\mH_{t}(\vtheta-\vtheta_{t}) + \text{const}.
        \end{aligned}
    \end{equation}
    Herein, since the constant term does not affect the optimization, minimizing the global objective $\gJ(\vtheta)$ is equivalent to minimizing the quadratic form.

    Recall that the Fr\'echet Mean objective defined in~\eqref{eq: frechet-objective} is essentially finding a point $\vtheta_{m}$ that minimizes the weighted sum of squared geodesic distances to all fine-tuned models. In a small neighborhood around each $\vtheta_t$, the squared geodesic distance $d_{\tG_{t}}^2(\theta, \theta_t)$ on a Riemannian manifold with metric tensor $\tG_{t} = \mH_t$ can be approximated by the quadratic form
    \begin{equation}
        d_{\tG_{t}}^2(\vtheta, \vtheta_t) \approx (\vtheta - \vtheta_t)^\top \mH_t (\vtheta - \vtheta_t)
    \end{equation}
    Substituting this approximation into the Fréchet Mean objective yields
    \begin{equation}
        \vtheta_m \approx \underset{\vtheta\in\Theta}{\arg\min} \sum_{t=1}^{T} \lambda_t (\vtheta - \vtheta_t)^\top \mH_t (\vtheta - \vtheta_t).
    \end{equation}
    Comparing the two approximations, it is obvious that the objective functions are identical (up to a constant scaling factor of $1/2$). Thus, minimizing the multi-task loss is locally equivalent to finding the Fréchet Mean on the manifold endowed with the Hessian metric.
\end{proof}

\subsection{Closed-Form Fr\'echet Mean \texorpdfstring{in~\eqref{eq: affine-sol}}{}}
Denote the objective in~\eqref{eq: affine-obj} by $\mathcal{J}(\boldsymbol{\delta}_{m})$. To minimize the objective, we take the first- and second-order derivatives with respect to $\boldsymbol{\delta}_{m}$,
\begin{equation}
    \frac{\partial}{\partial{\boldsymbol{\delta}_{m}}}\mathcal{J}(\boldsymbol{\delta}_{m})
    =2\sum\limits_{t=1}^{T}\lambda_{t}\mH_{t}\bigl(\boldsymbol{\delta}_{m} - \boldsymbol{\delta}_{t}\bigr),
    \label{eq: affine-obj-nabla}
\end{equation}
\begin{equation}
    \frac{\partial^{2}}{\partial\boldsymbol{\delta}_{m}^{2}}\mathcal{J}(\boldsymbol{\delta}_{m})
    =2\sum\limits_{t=1}^{T}\lambda_{t}\mH_{t}.
\end{equation}
Based on Assumption~\ref{asm: local-opt}, the Hessian evaluated at each local minimum is positive-definite, $\mH_{t}\succ{0}$. With non-negative scalar factors $\lambda_{t}\geq{0}$, the Hessian of $\mathcal{J}$ is positive-definite, making it strictly convex. Therefore, we can set~\eqref{eq: affine-obj-nabla} to zero and solve for the optimal $\boldsymbol{\delta}_{m}^{\ast}$ via
\begin{equation}
    \sum\limits_{t=1}^{T}\lambda_{t}\mH_{t}\boldsymbol{\delta}_{m}^{\ast}
    =\sum\limits_{t=1}^{T}\lambda_{t}\mH_{t}\boldsymbol{\delta}_{t},
\end{equation}
which yields the closed-form solution in~\eqref{eq: affine-sol}.

\subsection{Empirical Fisher Diagonal \texorpdfstring{in~\eqref{eq: hessian-approx}}{}}

Our goal is to approximate the expected Hessian $\mH_{t}$ at the converged parameters $\vtheta_{t}$. For models trained with a supervised log-likelihood objective, the expected Hessian coincides with the Fisher information matrix to second order, which equals the expected outer product of the per-sample gradient,
\begin{equation}
    \mH_{t}
    \;\approx\;\E_{\vx\sim\gD_{t}}\!\left[\nabla_{\vtheta}\gL_{t}(\vx,\vtheta_{t})\nabla_{\vtheta}\gL_{t}(\vx,\vtheta_{t})^{\top}\right].
\end{equation}
The off-diagonal entries of this $m\times m$ matrix are intractable to store at modern model scales. Hence, we propose approximate it by its diagonal,
\begin{equation}
    \mH_{t}\;\approx\;\text{diag}(\vv_{t}),\quad \vv_{t}\;=\;\E_{\vx\sim\gD_{t}}\!\bigl[(\nabla_{\vtheta}\gL_{t}(\vx,\vtheta_{t}))^{2}\bigr],
\end{equation}
which is exactly the empirical Fisher diagonal. The following~\Cref{appx: small-data} verifies that $\vv_{t}$ is robust to estimation from a tiny ($\leq 0.5\%$) subsample of the training set, so this proxy is essentially data-free in practice. The diagonal proxy by itself is degenerate in the full $m$-dimensional space due to the overparameterization of large neural network and that every per-task curvature shares the coordinate eigenbasis. Nevertheless,~\eqref{eq: projected-hessian-fisher} shows that projecting it onto the per-task tagged basis $\mS$ produces a \emph{dense} $p\times p$ matrix that recovers the off-diagonal curvature signal that the ambient diagonal discards.

\subsection{Proofs for the EpiMer Theory}
\label[appendix]{appx-subsec: subspace-proofs}

This subsection proves \cref{thm: existence,thm: error-bound,thm: curvature-advantage,prop: subsumption} stated in~\cref{subsec: theory}. Throughout, we work in the $p$-dimensional subspace defined by the orthonormal basis $\mS\in\mathbb{R}^{m\times p}$ ($\mS^{\top}\mS = \mI_p$) and write $\tilde{\boldsymbol{\delta}}_t = \mS^{\top}\boldsymbol{\delta}_t$, $\boldsymbol{\delta}_t^{\perp} = (\mI_m - \mS\mS^{\top})\boldsymbol{\delta}_t$, and $\bar{\mH} = \sum_{t=1}^T \lambda_t \tilde{\mH}_t$. We assume the merging coefficients are normalized so that $\sum_{t=1}^T \lambda_t = 1$, matching the convention used by every method we compare against (including our own experiments, which use $\lambda_t = 1/T$).

\subsubsection{Proof of Theorem~\ref{thm: existence} (Existence and Uniqueness)}

\begin{proof}
Expanding the merging objective gives a quadratic form in $\tilde{\boldsymbol{\delta}}$:
\begin{equation}
    \gF(\tilde{\boldsymbol{\delta}})
    = \tilde{\boldsymbol{\delta}}^{\top}\bar{\mH}\,\tilde{\boldsymbol{\delta}}
    \;-\; 2\,\tilde{\boldsymbol{\delta}}^{\top}\sum_{t=1}^T\lambda_t\tilde{\mH}_t\tilde{\boldsymbol{\delta}}_t
    \;+\; \underbrace{\sum_{t=1}^T\lambda_t\tilde{\boldsymbol{\delta}}_t^{\top}\tilde{\mH}_t\tilde{\boldsymbol{\delta}}_t}_{\text{constant in } \tilde{\boldsymbol{\delta}}}.
    \label{eq: frechet-quadratic}
\end{equation}
Differentiating twice yields $\nabla^2 \gF(\tilde{\boldsymbol{\delta}}) = 2\bar{\mH}$. By assumption $\bar{\mH}\succ 0$, so the Hessian is symmetric positive-definite and $\gF$ is strictly convex on $\mathbb{R}^p$. Consequently, $\gF$ admits a unique global minimizer, and that minimizer is the unique stationary point. Setting the gradient to zero,
\begin{equation}
    \nabla \gF(\tilde{\boldsymbol{\delta}}_m^*)
    = 2\bar{\mH}\,\tilde{\boldsymbol{\delta}}_m^* - 2\sum_{t=1}^T\lambda_t\tilde{\mH}_t\tilde{\boldsymbol{\delta}}_t = \vzero,
\end{equation}
and using invertibility of $\bar{\mH}$,
\begin{equation}
    \tilde{\boldsymbol{\delta}}_m^*
    = \bar{\mH}^{-1}\sum_{t=1}^T\lambda_t\tilde{\mH}_t\tilde{\boldsymbol{\delta}}_t
    = \left(\sum_{t=1}^T\lambda_t\tilde{\mH}_t\right)^{-1}\!\left[\sum_{t=1}^T\lambda_t\tilde{\mH}_t\tilde{\boldsymbol{\delta}}_t\right],
\end{equation}
which is exactly~\eqref{eq: subspace-frechet-mean}.
\end{proof}

\paragraph{Remark.} In our empirical-Fisher implementation we replace $\mH_t$ by $\text{diag}(\vv_t) + \epsilon\mI_m$ with a small jitter $\epsilon > 0$ before projecting (see \texttt{ensure\_psd} in \texttt{hessian.py}). This guarantees $\tilde{\mH}_t \succ 0$ for every task and hence $\bar{\mH}\succ 0$, so the existence/uniqueness hypothesis holds unconditionally in practice.

\subsubsection{Proof of Theorem~\ref{thm: error-bound} (Merging Error Bound)}

\begin{proof}
By Assumption~\ref{asm: smooth} the loss $\gL_t$ is $C^2$, so a second-order Taylor expansion of $\gL_t$ around $\vtheta_t$ gives
\begin{equation}
    \begin{aligned}
    &\gL_t(\vtheta_m) - \gL_t(\vtheta_t) =\\
    &\nabla\gL_t(\vtheta_t)^{\top}(\vtheta_m - \vtheta_t)
    + \tfrac{1}{2}(\vtheta_m - \vtheta_t)^{\top}\mH_t(\vtheta_m - \vtheta_t)
    + R_3^{(t)},
    \end{aligned}
\end{equation}
with cubic remainder $R_3^{(t)} = \gO(\|\vtheta_m - \vtheta_t\|^3)$. \Cref{asm: local-opt} kills the linear term, leaving
\begin{equation}
    \gL_t(\vtheta_m) - \gL_t(\vtheta_t)
    = \tfrac{1}{2}(\vtheta_m - \vtheta_t)^{\top}\mH_t(\vtheta_m - \vtheta_t) + R_3^{(t)}.
    \label{eq: per-task-quadratic}
\end{equation}
Then, we can decompose $\vtheta_m - \vtheta_t$ along the subspace $\mS$ and its orthogonal complement. Since $\vtheta_m = \vtheta_0 + \mS\tilde{\boldsymbol{\delta}}_m^*$ and $\vtheta_t = \vtheta_0 + \boldsymbol{\delta}_t = \vtheta_0 + \mS\tilde{\boldsymbol{\delta}}_t + \boldsymbol{\delta}_t^{\perp}$, we have
\begin{equation}
    \vtheta_m - \vtheta_t
    = \mS\bigl(\tilde{\boldsymbol{\delta}}_m^* - \tilde{\boldsymbol{\delta}}_t\bigr) - \boldsymbol{\delta}_t^{\perp}.
    \label{eq: parallel-perp-decomp}
\end{equation}
Let $\Delta_t \triangleq \tilde{\boldsymbol{\delta}}_m^* - \tilde{\boldsymbol{\delta}}_t$. Substituting~\eqref{eq: parallel-perp-decomp} into the quadratic form and using $\mS^{\top}\mH_t\mS = \tilde{\mH}_t$,
\begin{equation}
\begin{aligned}
    &(\vtheta_m - \vtheta_t)^{\top}\mH_t(\vtheta_m - \vtheta_t)\\
    &= \Delta_t^{\top}\tilde{\mH}_t\Delta_t
    - 2\,\Delta_t^{\top}\mS^{\top}\mH_t\boldsymbol{\delta}_t^{\perp}
    + (\boldsymbol{\delta}_t^{\perp})^{\top}\mH_t\boldsymbol{\delta}_t^{\perp}.
\end{aligned}
\end{equation}
The first and third terms are exactly the per-task contributions to $\gV_S$ and $\gR_S$. The cross term we bound via two applications of Cauchy--Schwarz. First, the $\mH_t$-Cauchy--Schwarz inequality (valid because $\mH_t\succeq 0$),
\begin{equation}
\begin{aligned}
    &\bigl|\Delta_t^{\top}\mS^{\top}\mH_t\boldsymbol{\delta}_t^{\perp}\bigr|=\\
    &\bigl|(\mS\Delta_t)^{\top}\mH_t\boldsymbol{\delta}_t^{\perp}\bigr|
    \leq \sqrt{\Delta_t^{\top}\tilde{\mH}_t\Delta_t}\;\sqrt{(\boldsymbol{\delta}_t^{\perp})^{\top}\mH_t\boldsymbol{\delta}_t^{\perp}}.
\end{aligned}
\end{equation}
Multiplying by $\lambda_t$, summing over $t$, and applying the standard Cauchy--Schwarz inequality on the resulting sum,
\begin{equation}
\begin{aligned}
    \sum_{t=1}^T\lambda_t \bigl|\Delta_t^{\top}\mS^{\top}\mH_t\boldsymbol{\delta}_t^{\perp}\bigr|
    &\leq \sum_{t=1}^T \sqrt{\lambda_t\,\Delta_t^{\top}\tilde{\mH}_t\Delta_t}\;\sqrt{\lambda_t\,(\boldsymbol{\delta}_t^{\perp})^{\top}\mH_t\boldsymbol{\delta}_t^{\perp}} \\
    &\leq \sqrt{\sum_{t=1}^T \lambda_t\,\Delta_t^{\top}\tilde{\mH}_t\Delta_t}\,\sqrt{\sum_{t=1}^T \lambda_t\,(\boldsymbol{\delta}_t^{\perp})^{\top}\mH_t\boldsymbol{\delta}_t^{\perp}} \\
    &= \sqrt{\gV_S}\,\sqrt{\gR_S}.
\end{aligned}
\end{equation}
Combining the all three terms above, we have
\begin{equation}
\begin{aligned}
    \sum_{t=1}^T \lambda_t (\vtheta_m - \vtheta_t)^{\top}\mH_t(\vtheta_m - \vtheta_t)
    &\leq \gV_S + 2\sqrt{\gV_S}\sqrt{\gR_S} + \gR_S \\
    &= \bigl(\sqrt{\gV_S} + \sqrt{\gR_S}\bigr)^2.
\end{aligned}
\end{equation}
Plugging this into the weighted sum of~\eqref{eq: per-task-quadratic} and absorbing the per-task remainders into $R_3 = \sum_{t}\lambda_t R_3^{(t)} = \gO(\max_t \|\vtheta_m - \vtheta_t\|^3)$ yields
\begin{equation}
    \sum_{t=1}^T\lambda_t\bigl[\gL_t(\vtheta_m) - \gL_t(\vtheta_t)\bigr]
    \leq \tfrac{1}{2}\bigl(\sqrt{\gV_S} + \sqrt{\gR_S}\bigr)^2 + R_3,
\end{equation}
which is~\eqref{eq: error-bound}.
\end{proof}

\paragraph{Remark.} Equality in the bound is attained when, for every $t$ with $\lambda_t > 0$, the parallel error $\mS\Delta_t$ and the residual $\boldsymbol{\delta}_t^{\perp}$ are co-linear under the metric $\mH_t$ and the per-task ratios $\sqrt{\Delta_t^{\top}\tilde{\mH}_t\Delta_t} / \sqrt{(\boldsymbol{\delta}_t^{\perp})^{\top}\mH_t\boldsymbol{\delta}_t^{\perp}}$ are constant. The bound is therefore tight, not just an order-of-magnitude estimate.

\subsubsection{Proof of Theorem~\ref{thm: curvature-advantage} (Curvature Advantage)}

\begin{proof}
Let $\vb \triangleq \sum_{t=1}^T \lambda_t \tilde{\mH}_t \tilde{\boldsymbol{\delta}}_t$, so the curvature-aware Fr\'echet mean of \cref{thm: existence} is $\tilde{\boldsymbol{\delta}}_H = \bar{\mH}^{-1}\vb$. From the quadratic expansion in \eqref{eq: frechet-quadratic},
\begin{equation}
    \gF(\tilde{\boldsymbol{\delta}}_H)
    = \vb^{\top}\bar{\mH}^{-1}\bar{\mH}\,\bar{\mH}^{-1}\vb - 2\vb^{\top}\bar{\mH}^{-1}\vb + C
    = -\vb^{\top}\bar{\mH}^{-1}\vb + C,
    \label{eq: F-at-H}
\end{equation}
where $C \triangleq \sum_t \lambda_t \tilde{\boldsymbol{\delta}}_t^{\top}\tilde{\mH}_t \tilde{\boldsymbol{\delta}}_t$ is independent of $\tilde{\boldsymbol{\delta}}$. Substituting the flat-geometry choice $\tilde{\boldsymbol{\delta}}_I = \bar{\boldsymbol{\delta}} = \sum_t\lambda_t\tilde{\boldsymbol{\delta}}_t$,
\begin{equation}
    \gF(\tilde{\boldsymbol{\delta}}_I)
    = \bar{\boldsymbol{\delta}}^{\top}\bar{\mH}\,\bar{\boldsymbol{\delta}} - 2\bar{\boldsymbol{\delta}}^{\top}\vb + C.
    \label{eq: F-at-I}
\end{equation}
Subtracting~\eqref{eq: F-at-H} from~\eqref{eq: F-at-I},
\begin{equation}
\begin{aligned}
    \gF(\tilde{\boldsymbol{\delta}}_I) - \gF(\tilde{\boldsymbol{\delta}}_H)
    &= \bar{\boldsymbol{\delta}}^{\top}\bar{\mH}\,\bar{\boldsymbol{\delta}} - 2\bar{\boldsymbol{\delta}}^{\top}\vb + \vb^{\top}\bar{\mH}^{-1}\vb \\
    &= \bigl(\bar{\mH}\,\bar{\boldsymbol{\delta}} - \vb\bigr)^{\top}\bar{\mH}^{-1}\bigl(\bar{\mH}\,\bar{\boldsymbol{\delta}} - \vb\bigr),
\end{aligned}
\end{equation}
where the second equality is the identity $\vx^{\top}\mA\vx - 2\vx^{\top}\vy + \vy^{\top}\mA^{-1}\vy = (\mA\vx - \vy)^{\top}\mA^{-1}(\mA\vx - \vy)$ valid for any $\mA\succ 0$.

It remains to show that $\bar{\mH}\,\bar{\boldsymbol{\delta}} - \vb = -\vc$ where $\vc = \sum_t \lambda_t (\tilde{\mH}_t - \bar{\mH})(\tilde{\boldsymbol{\delta}}_t - \bar{\boldsymbol{\delta}})$. Expanding $\vc$ and using $\sum_t\lambda_t = 1$,
\begin{equation}
\begin{aligned}
    \vc
    &= \sum_{t=1}^T \lambda_t \tilde{\mH}_t\tilde{\boldsymbol{\delta}}_t
     - \Bigl(\sum_{t=1}^T \lambda_t \tilde{\mH}_t\Bigr)\bar{\boldsymbol{\delta}}
     - \bar{\mH}\sum_{t=1}^T \lambda_t \tilde{\boldsymbol{\delta}}_t
     + \bar{\mH}\,\bar{\boldsymbol{\delta}}\sum_{t=1}^T\lambda_t \\
    &= \vb - \bar{\mH}\,\bar{\boldsymbol{\delta}} - \bar{\mH}\,\bar{\boldsymbol{\delta}} + \bar{\mH}\,\bar{\boldsymbol{\delta}}
     = \vb - \bar{\mH}\,\bar{\boldsymbol{\delta}}.
\end{aligned}
\end{equation}
Hence $\bar{\mH}\,\bar{\boldsymbol{\delta}} - \vb = -\vc$, and the quadratic form is invariant under sign,
\begin{equation}
    \gF(\tilde{\boldsymbol{\delta}}_I) - \gF(\tilde{\boldsymbol{\delta}}_H)
    = \vc^{\top}\bar{\mH}^{-1}\vc \;\geq\; 0,
\end{equation}
where non-negativity follows from $\bar{\mH}^{-1}\succ 0$. Equality holds if and only if $\vc = \vzero$, \emph{i.e.~iff the curvature deviations and task-vector deviations are uncorrelated under the weighting $\{\lambda_{t}\mid{t=1,\ldots,T}\}$}.
\end{proof}

\paragraph{Remark.} Three sufficient conditions for $\vc = \vzero$ are immediate: \emph{(a)} all $\tilde{\mH}_t$ are equal (so $\tilde{\mH}_t - \bar{\mH} = \vzero$); \emph{(b)} all $\tilde{\boldsymbol{\delta}}_t$ are equal (so $\tilde{\boldsymbol{\delta}}_t - \bar{\boldsymbol{\delta}} = \vzero$); \emph{(c)} the first moment of the deviations $(\tilde{\mH}_t - \bar{\mH})(\tilde{\boldsymbol{\delta}}_t - \bar{\boldsymbol{\delta}})$ vanishes by symmetry. Cases (a) and (b) are degenerate (homogeneous curvature, identical task vectors); case (c) is the empirically interesting one and is precisely what the diagnostic $\eta = \vc^{\top}\bar{\mH}^{-1}\vc$ measures, evaluated in the projected $p\times p$ space at $\gO(p^3)$ cost.

\subsubsection{Proof of Proposition~\ref{prop: subsumption} (Subsumptions)}

\begin{proof}
We verify each of the five cases by direct substitution into the closed-form Fr\'echet mean
\begin{equation}
    \tilde{\boldsymbol{\delta}}_m^*
    = \Bigl(\sum_{t=1}^T \lambda_t \tilde{\mH}_t\Bigr)^{-1}\!\Bigl[\sum_{t=1}^T \lambda_t \tilde{\mH}_t \tilde{\boldsymbol{\delta}}_t\Bigr],
    \quad \vtheta_m = \vtheta_0 + \mS\,\tilde{\boldsymbol{\delta}}_m^*,
    \label{eq: appx-subspace-frechet}
\end{equation}
under the convention $\sum_{t=1}^T \lambda_t = 1$.

\paragraph{Case 1 (Task Arithmetic).} Set $\mS = \mI_m$ and $\tilde{\mH}_t = \mI_m$. Then $\tilde{\boldsymbol{\delta}}_t = \boldsymbol{\delta}_t$ and~\eqref{eq: appx-subspace-frechet} reduces to
\begin{equation}
    \tilde{\boldsymbol{\delta}}_m^* = \Bigl(\sum_t \lambda_t\Bigr)^{-1}\sum_t\lambda_t\boldsymbol{\delta}_t = \sum_t \lambda_t\boldsymbol{\delta}_t,
\end{equation}
so $\vtheta_m = \vtheta_0 + \sum_t \lambda_t \boldsymbol{\delta}_t$, which is exactly the Task Arithmetic update~\cite{ilharco2023editingmodelstaskarithmetic}.

\paragraph{Case 2 (Fisher Averaging).} Set $\mS = \mI_m$ and $\tilde{\mH}_t = \text{diag}(\mF_t)$. Both $\sum_t \lambda_t \text{diag}(\mF_t)$ and $\sum_t \lambda_t \text{diag}(\mF_t) \boldsymbol{\delta}_t$ are coordinate-wise diagonal, so~\eqref{eq: appx-subspace-frechet} gives
\begin{equation}
    (\tilde{\boldsymbol{\delta}}_m^*)_i
    = \frac{\sum_{t=1}^T \lambda_t F_{t,i}\,\delta_{t,i}}{\sum_{t=1}^T \lambda_t F_{t,i}}, \qquad i=1,\ldots,m,
\end{equation}
which is the per-coordinate Fisher-weighted average of the task vectors~\cite{NEURIPS2022_70c26937}.

\paragraph{Case 3 (Daheim et al.).} Set $\mS = \mI_m$ and $\tilde{\mH}_t = \mH_t$ (full Hessian, possibly approximated). Then
\begin{equation}
    \tilde{\boldsymbol{\delta}}_m^* = \Bigl(\sum_t \lambda_t \mH_t\Bigr)^{-1}\sum_t\lambda_t\mH_t\boldsymbol{\delta}_t,
\end{equation}
which yields the full-space curvature-aware merge of~\cite{daheim2024modelmerginguncertaintybasedgradient}.

\paragraph{Case 4 (TSV-M, projection scaffold).} Let $\mS$ be the orthonormal basis for the top-$k$ left singular subspace of the stacked task vectors $[\boldsymbol{\delta}_1,\ldots,\boldsymbol{\delta}_T]$ (the joint-SVD construction of \cite{gargiulo2025tasksingularvectorsreducing}), and set $\tilde{\mH}_t = \mI_p$. Then~\eqref{eq: appx-subspace-frechet} simplifies to $\tilde{\boldsymbol{\delta}}_m^* = \sum_t \lambda_t \tilde{\boldsymbol{\delta}}_t = \sum_t \lambda_t \mS^{\top}\boldsymbol{\delta}_t$, and lifting back to ambient space yields
\begin{equation}
    \vtheta_m = \vtheta_0 + \mS\mS^{\top}\Bigl(\sum_t \lambda_t \boldsymbol{\delta}_t\Bigr),
    \label{eq: appx-tsvm-projection}
\end{equation}
that is, Task Arithmetic projected onto the rank-$k$ task-singular subspace. This captures the \emph{subspace-projection scaffold} underlying TSV-M. Algorithm~1 of \cite{gargiulo2025tasksingularvectorsreducing} additionally performs a per-layer orthogonal Procrustes whitening of the stacked per-task singular vectors.
This decorrelation is nonlinear in $\{\boldsymbol{\delta}_t\mid{t=1,\ldots,T}\}$ and therefore lies outside the quadratic Fr\'echet template of~\eqref{eq: appx-subspace-frechet}. Accordingly, our hierarchy subsumes TSV-M's curvature-agnostic subspace-projection choice but leaves the Procrustes decorrelation as an orthogonal enhancement. Empirically we still compare against Gargiulo's complete algorithm in~\cref{sec: experiment}.

\paragraph{Case 5 (EpiMer).} EpiMer adopts the same per-task factored basis as TSV-M (\eqref{eq: subspace-basis}), but replaces the identity metric $\tilde{\mH}_t = \mI_p$ of Case~4 with the projected curvature $\tilde{\mH}_t = \mS^{\top}\mH_t\mS$, where $\mS$ collects the vec'd rank-1 atoms $\{\mU_{\perp,i}\mV_{\perp,i}^{\top}\}$ as columns. Substituting this curvature-aware metric into the quadratic Fr\'echet template~\eqref{eq: appx-subspace-frechet} recovers \eqref{eq: subspace-frechet-mean} verbatim. On the per-task tagged basis, however, the Fr\'echet \emph{mean} can suffer from a magnitude collapse pathology because each rank-1 atom is owned by exactly one task. In practice, our proposed EpiMer in~\cref{subsubsec: aggregator} uses the matrix-weighted sum aggregator of~\eqref{eq: frechet-sum}, which is exactly $\alpha T$ times the Fr\'echet mean and reshapes the per-task contributions through the same curvature solve. In the homogeneous-curvature limit $\tilde{\mH}_t \equiv \tilde{\mH}$, \eqref{eq: frechet-sum} collapses to $\alpha\sum_t\tilde{\boldsymbol{\delta}}_t$, recovering TSV-M aggregation. Nevertheless, whenever the per-task curvatures differ, our matrix solve can yield a strictly different and empirically better merge. See~\cref{appx: alpha-sweep} for more detailed comparison.

The above shows that in each case the existing method
is recovered by an appropriate $(\mS,\tilde{\mH}_t)$. This completes the proof.
\end{proof}


\begin{table*}[!t]
\centering
\caption{Per-task top-$1$ accuracy on the eight image classification tasks, expanded from \cref{tab: performance}. Best among merging methods is \textbf{bold}, second best is \underline{underlined}. All baselines use $\lambda_t = 1/T$, so AM and TA collapse to the same merged delta and are listed jointly. TSV-M and EpiMer are reported at $k=32$ and the per-method optimal global rescaling $\alpha$ from~\cref{subsec: alpha-sweep-main}: TSV-M uses $\alpha=0.70$ on ViT-B/32 and $\alpha=1.0$ on ViT-B/16 and ViT-L/14; EpiMer uses $\alpha=1.0$ on all three backbones. On ViT-L/14, the average accuracy of TSV-M and EpiMer differ by $0.06$ percentage points (EpiMer $0.9065$ vs TSV-M $0.9059$); both display as $.906$ at the table's three-decimal precision.}
\resizebox{\textwidth}{!}{%
\begin{tabular}{@{}ll|ccccccccc@{}}
\toprule
\textbf{Backbone} & \textbf{Methods} & \textbf{Cars} & \textbf{DTD} & \textbf{EuroSAT} & \textbf{GTSRB} & \textbf{MNIST} & \textbf{RESISC45} & \textbf{SUN397} & \textbf{SVHN} & \textbf{Avg.} \\
\midrule
\multicolumn{11}{l}{\textit{Baselines}} \\
\multirow{2}{*}{ViT-B/32} & Linear-Probing & .597 & .436 & .451 & .326 & .482 & .602 & .631 & .316 & .480 \\
& Fine-tuning      & .806 & .791 & .999 & .990 & .997 & .961 & .755 & .975 & .909 \\
\multirow{2}{*}{ViT-B/16} & Linear-Probing & .647 & .441 & .544 & .434 & .517 & .663 & .655 & .520 & .553 \\
& Fine-tuning      & .891 & .822 & .999 & .992 & .998 & .968 & .785 & .978 & .929 \\
\multirow{2}{*}{ViT-L/14} & Linear-Probing & .779 & .549 & .613 & .507 & .763 & .713 & .682 & .584 & .649 \\
& Fine-tuning      & .934 & .839 & .999 & .993 & .997 & .973 & .824 & .981 & .943 \\
\midrule
\multicolumn{11}{l}{\textit{Flat-geometry methods}} \\
\multirow{3}{*}{ViT-B/32} & AM~\cite{Wortsman_2022_CVPR} / TA~\cite{ilharco2023editingmodelstaskarithmetic}   & .575 & .506 & .729 & .529 & .871 & .715 & .650 & .647 & .653 \\
                           & TIES~\cite{NEURIPS2023_1644c9af}                                                 & .572 & .532 & .871 & .741 & .983 & .696 & .569 & .835 & .725 \\
                           & TSV-M~\cite{gargiulo2025tasksingularvectorsreducing}                             & .675 & .657 & .959 & .881 & .990 & .835 & .674 & .903 & \underline{.822} \\
\cmidrule{2-11}
\multirow{3}{*}{ViT-B/16} & AM~\cite{Wortsman_2022_CVPR} / TA~\cite{ilharco2023editingmodelstaskarithmetic}   & .672 & .504 & .760 & .599 & .941 & .759 & .677 & .765 & .710 \\
                           & TIES~\cite{NEURIPS2023_1644c9af}                                                 & .693 & .559 & .838 & .804 & .989 & .778 & .637 & .896 & .774 \\
                           & TSV-M~\cite{gargiulo2025tasksingularvectorsreducing}                             & .802 & .711 & .977 & .937 & .994 & .860 & .686 & .955 & \underline{.865} \\
\cmidrule{2-11}
\multirow{3}{*}{ViT-L/14} & AM~\cite{Wortsman_2022_CVPR} / TA~\cite{ilharco2023editingmodelstaskarithmetic}   & .797 & .619 & .901 & .713 & .967 & .826 & .715 & .786 & .791 \\
                           & TIES~\cite{NEURIPS2023_1644c9af}                                                 & .845 & .675 & .950 & .907 & .991 & .873 & .737 & .890 & .859 \\
                           & TSV-M~\cite{gargiulo2025tasksingularvectorsreducing}                             & .899 & .773 & .985 & .965 & .995 & .909 & .762 & .960 & \underline{.906} \\
\midrule
\multicolumn{11}{l}{\textit{Curvature-aware methods}} \\
\multirow{2}{*}{ViT-B/32} & Fisher~\cite{NEURIPS2022_70c26937}  & .751 & .598 & .409 & .338 & .542 & .620 & .707 & .351 & .539 \\
                           & EpiMer                     & .672 & .694 & .965 & .914 & .993 & .834 & .665 & .926 & \textbf{.833} \\
\cmidrule{2-11}
\multirow{2}{*}{ViT-B/16} & Fisher~\cite{NEURIPS2022_70c26937}  & .824 & .610 & .503 & .408 & .678 & .688 & .740 & .547 & .625 \\
                           & EpiMer                     & .804 & .722 & .977 & .944 & .993 & .869 & .710 & .941 & \textbf{.870} \\
\cmidrule{2-11}
\multirow{2}{*}{ViT-L/14} & Fisher~\cite{NEURIPS2022_70c26937}  & .874 & .760 & .646 & .507 & .755 & .760 & .788 & .671 & .720 \\
                           & EpiMer                     & .901 & .771 & .984 & .973 & .995 & .914 & .763 & .949 & \textbf{.906} \\
\bottomrule
\end{tabular}}
\label{tab: performance-full}
\end{table*}

\section{Implementation Details for Experiments}
\label[appendix]{appx: implementation}

\subsection{Datasets}
\label[appendix]{appx-subsec: datasets}

This section presents more details about the datasets used in our comparative experiments. Following prior works~\cite{Wortsman_2022_CVPR,pmlr-v162-wortsman22a,ilharco2023editingmodelstaskarithmetic}, we evaluated model merging on eight image classification tasks:
\begin{itemize}
    \item \textbf{Stanford Cars (Cars)~\cite{KrauseStarkDengFei-Fei_3DRR2013}} is an image classification dataset with $16,185$ images of $196$ classes of cars. The dataset is split into $8,144$ training images and $8,041$ testing images. Each class appears with the exact frequency in the training and testing sets.
    \item \textbf{Describable Texture Dataset (DTD)~\cite{cimpoi14describing}} comprises of $3,760$ training images and $1,880$ testing image.
    Each image is labeled with one of $47$ describable textures.
    \item \textbf{EuroSAT~\cite{helber2019eurosatnoveldatasetdeep}} is a dataset for land use classification with $27,000$ labeled Sentinel-2 satellite images. The dataset is split into $21,000$ training and $6,000$ testing images.
    \item \textbf{German Traffic Sign Recognition Benchmark Dataset (GTSRB)~\cite{6033395}} contains $39,270$ images with $43$ classes of traffic signs. The dataset is split into $26,640$ training images and $12,630$ testing images.
    \item \textbf{MNIST~\cite{lecun2010mnist}} is a dataset of $10$ handwritten digits, containing $60,000$ training images and $10,000$ testing images, with balanced presence of each digit in both splits.
    \item \textbf{Remote Sensing Image Scene Classification Dataset (RESISC45)~\cite{Cheng_2017}} is a remote sensing image classification dataset with $25,200$ images of $45$ classes of scenarios. The dataset is split into $18,900$ training images and $6,300$ testing images, with around $700$ images per class.
    \item \textbf{Scene Understanding Benchmark (SUN397)~\cite{sun397}} contains $397$ categories of scenes with $108,754$ images in total. The number of images varies across categories, but there are at least $100$ images per category.
    \item \textbf{Street View House Numbers (SVHN)~\cite{netzer2011reading}} is a dataset of $10$ classes of housing number digits extracted from the Google Street View images. The dataset contains $73,257$ training images and $26,032$ testing images.
\end{itemize}

\subsection{Model Checkpoints}
\label[appendix]{appx-subsec: finetuning}

We use the publicly released, per-task fine-tuned CLIP-ViT-B/16, CLIP-ViT-B/32, and CLIP-ViT-L/14 checkpoints from~\cite{ilharco2023editingmodelstaskarithmetic}, which were obtained by linear probing followed by full fine-tuning on each of the eight datasets in~\cref{appx-subsec: datasets}. The checkpoints ship only the converged parameters $\vtheta_{t}$, not the optimizer states; we therefore reconstruct the empirical Fisher diagonal $\vv_{t}$ used by~\eqref{eq: hessian-approx} via a single forward+backward pass over each task's training set at $\vtheta_{t}$. Refer to~\cref{appx: small-data} for more details.

\subsection{Baselines Implementation}
\label[appendix]{appx-subsec: baselines}

For the baseline methods, we implement Arithmetic Mean (AM)~\cite{Wortsman_2022_CVPR}, Task Arithmetic~\cite{ilharco2023editingmodelstaskarithmetic}, Fisher-weighted Averaging~\cite{NEURIPS2022_70c26937}, and TIES-Merging~\cite{NEURIPS2023_1644c9af}. We follow the original papers for the implementation details of these methods. In particular, we use equal merging coefficients $\lambda_{t}=0.125$ across all tasks for Task Arithmetic and TIES-Merging. We keep the top-$20\%$ of the parameters and elect signs by frequency in TIES-Merging. To facilitate a fair comparison, we leverage the square of the first momentum to construct a diagonal Fisher Information Matrix $\mF_{t}\gets\text{diag}(\vu^{2})$ in Fisher-weighted Averaging, following the definitions in the original paper.

\section{Additional Results}
\label[appendix]{appx: results}

\subsection{Per-Task Accuracy Breakdown}
\label[appendix]{appx: performance-full}

\Cref{tab: performance-full} expands~\cref{tab: performance} from the main text with the per-task top-$1$ accuracy on every dataset, for all baselines and EpiMer, across the three CLIP-ViT backbones. Averages in the last column match the main-text table exactly.

\subsection{Global Rescaling \texorpdfstring{$\alpha$}{alpha} Sensitivity}
\label[appendix]{appx: alpha-sweep}

This section reports the full $\alpha$ sweep at $k\in\{4,16,32\}$ on each backbone and supplements the $k=32$ slice that appears in~\cref{subsec: alpha-sweep-main}. We sweep $\alpha\in\{0.20,0.30,0.40,0.50,0.70,1.00\}$ on TSV-M~\cite{gargiulo2025tasksingularvectorsreducing} and the matrix-weighted aggregator of~\eqref{eq: frechet-sum}; \cref{tab: alpha-sweep-b32,tab: alpha-sweep-b16,tab: alpha-sweep-l14} report average accuracy across the eight tasks.


\begin{table}[!t]
    \centering
    \caption{ViT-B/32 $\alpha$-sensitivity sweep. Best per column is \textbf{bold}.}
    \resizebox{\linewidth}{!}{%
    \begin{tabular}{@{}lcccccc@{}}
    \toprule
    & \multicolumn{3}{c}{\textbf{TSV-M}} & \multicolumn{3}{c}{\textbf{EpiMer}} \\
    \cmidrule(lr){2-4} \cmidrule(lr){5-7}
    $\alpha$ & $k=4$ & $k=16$ & $k=32$ & $k=4$ & $k=16$ & $k=32$ \\
    \midrule
    0.20 & .573 & .624 & .630 & .573 & .609 & .601 \\
    0.30 & .627 & .689 & .699 & .633 & .677 & .670 \\
    0.40 & .667 & .737 & .750 & .674 & .727 & .724 \\
    0.50 & .693 & .770 & .787 & .703 & .763 & .764 \\
    0.70 & \textbf{.720} & \textbf{.800} & \textbf{.822} & .728 & .800 & .812 \\
    1.00 & .706 & .788 & .820 & \textbf{.709} & \textbf{.802} & \textbf{.833} \\
    \bottomrule
    \end{tabular}}
    \label{tab: alpha-sweep-b32}
\end{table}

\begin{table}[!t]
    \centering
    \caption{ViT-B/16 $\alpha$-sensitivity sweep. Best per column is \textbf{bold}.}
    \resizebox{\linewidth}{!}{%
    \begin{tabular}{@{}lcccccc@{}}
    \toprule
    & \multicolumn{3}{c}{\textbf{TSV-M}} & \multicolumn{3}{c}{\textbf{EpiMer}} \\
    \cmidrule(lr){2-4} \cmidrule(lr){5-7}
    $\alpha$ & $k=4$ & $k=16$ & $k=32$ & $k=4$ & $k=16$ & $k=32$ \\
    \midrule
    0.20 & .636 & .681 & .688 & .638 & .672 & .666 \\
    0.30 & .676 & .733 & .747 & .680 & .725 & .724 \\
    0.40 & .703 & .774 & .792 & .709 & .766 & .771 \\
    0.50 & .725 & .803 & .822 & .730 & .796 & .804 \\
    0.70 & \textbf{.755} & .834 & .857 & \textbf{.760} & .832 & .846 \\
    1.00 & .754 & \textbf{.836} & \textbf{.865} & .751 & \textbf{.843} & \textbf{.870} \\
    \bottomrule
    \end{tabular}}
    \label{tab: alpha-sweep-b16}
\end{table}

\begin{table}[!t]
    \centering
    \caption{ViT-L/14 $\alpha$-sensitivity sweep. Best per column is \textbf{bold}.}
    \resizebox{\linewidth}{!}{%
    \begin{tabular}{@{}lcccccc@{}}
    \toprule
    & \multicolumn{3}{c}{\textbf{TSV-M}} & \multicolumn{3}{c}{\textbf{EpiMer}} \\
    \cmidrule(lr){2-4} \cmidrule(lr){5-7}
    $\alpha$ & $k=4$ & $k=16$ & $k=32$ & $k=4$ & $k=16$ & $k=32$ \\
    \midrule
    0.20 & .730 & .763 & .772 & .733 & .761 & .766 \\
    0.30 & .759 & .801 & .816 & .764 & .801 & .808 \\
    0.40 & .783 & .830 & .849 & .789 & .830 & .841 \\
    0.50 & .801 & .850 & .870 & .809 & .849 & .863 \\
    0.70 & .826 & .875 & .895 & .832 & .875 & .890 \\
    1.00 & \textbf{.837} & \textbf{.887} & \textbf{.906} & \textbf{.842} & \textbf{.890} & \textbf{.906} \\
    \bottomrule
    \end{tabular}}
    \label{tab: alpha-sweep-l14}
\end{table}

Beyond the $k=32$ findings reported in~\cref{subsec: alpha-sweep-main}, the lower-rank slices in \cref{tab: alpha-sweep-b32,tab: alpha-sweep-b16,tab: alpha-sweep-l14} establish that EpiMer dominates TSV-M at every rank tested at the per-method optimum: by $+0.79$, $+0.20$, $+1.10$ percentage points at $k=4,16,32$ on ViT-B/32; by $+0.52$, $+0.66$, $+0.48$ percentage points on ViT-B/16; and by $+0.49$, $+0.25$, $+0.06$ percentage points on ViT-L/14. The L/14 margin shrinks at high rank because the per-task curvatures become more uniform on the larger backbone, pushing~\eqref{eq: frechet-sum} toward its homogeneous-curvature limit (TSV-M); even there, EpiMer remains strictly ahead.

\subsection{Loss Landscape Visualizations}
\label[appendix]{appx: loss-landscape}

To support the connected-basin assumption (Assumption~\ref{asm: connected-basin}) used by EpiMer, we visualize the loss landscape of fine-tuned CLIP-ViT models around their fine-tuned positions. Following prior works~\cite{NEURIPS2018_a41b3bb3,pmlr-v222-ding24a}, we apply PCA on the matrix $\bm{\Theta}$ that stacks the parameter vectors of all fine-tuned models for a given backbone, take the top two principal directions $\delta,\eta$, and evaluate the per-task loss on the grid $\vtheta = \bar{\vtheta} + \alpha\delta + \beta\eta$, where $\bar{\vtheta}$ is the mean of all fine-tuned models. \Cref{fig: clip-vit-b16 additional landscape} reports this visualization for CLIP-ViT-B/16 on DTD, EuroSAT, GTSRB, and SUN397, all of which exhibit a low-change-of-loss basin around the fine-tuned model with high-loss barriers between the per-task minima---qualitatively consistent with Assumption~\ref{asm: connected-basin}.

\begin{figure*}[!t]
    \centering
    \begin{subfigure}[b]{0.48\textwidth}
        \includegraphics[width=\textwidth]{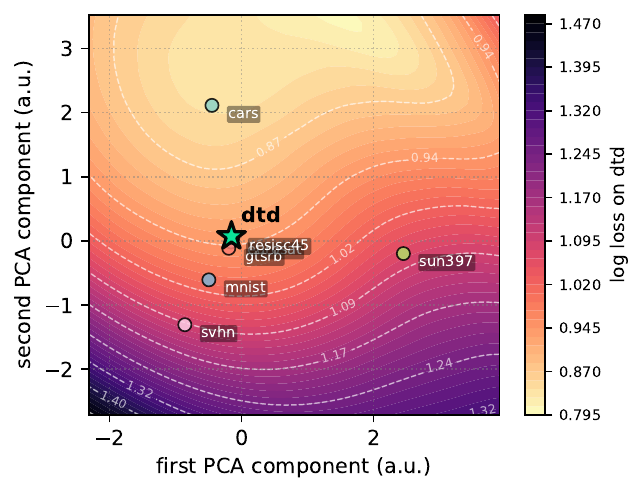}
        \caption{CLIP-ViT-B/16 on DTD.}
    \end{subfigure}
    \hfill 
    \begin{subfigure}[b]{0.48\textwidth}
        \includegraphics[width=\textwidth]{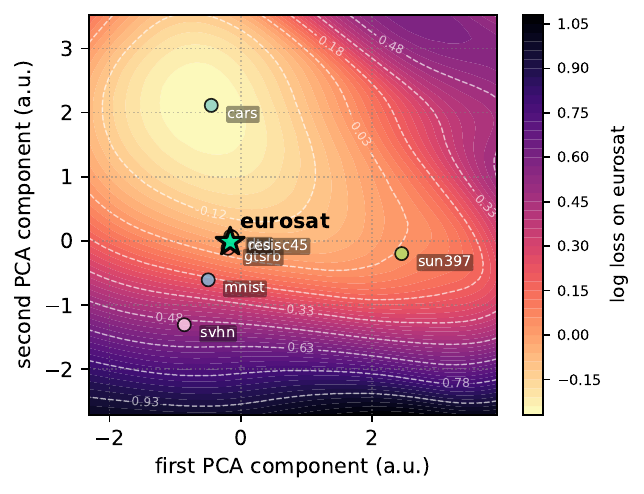}
        \caption{CLIP-ViT-B/16 on EuroSAT.}
    \end{subfigure}
    \vfill
    \begin{subfigure}[b]{0.48\textwidth}
        \includegraphics[width=\textwidth]{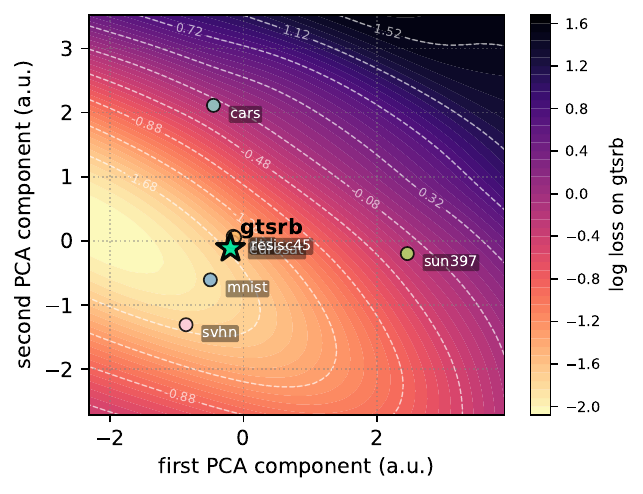}
        \caption{CLIP-ViT-B/16 on GTSRB.}
    \end{subfigure}
    \hfill
    \begin{subfigure}[b]{0.48\textwidth}
        \includegraphics[width=\textwidth]{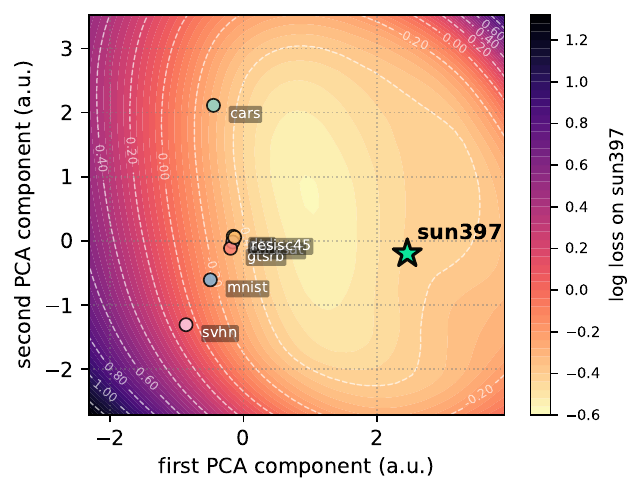}
        \caption{CLIP-ViT-B/16 on SUN397.}
    \end{subfigure}
    \caption{Loss landscape visualization for CLIP-ViT-B/16 fine-tuned on (a) DTD, (b) EuroSAT, (c) GTSRB, and (d) SUN397 datasets.}
    \label{fig: clip-vit-b16 additional landscape}
\end{figure*}

\subsection{Data-Efficiency for Empirical Fisher Diagonal}
\label[appendix]{appx: small-data}

This section reports the full setup used by the empirical Fisher robustness study summarized in~\cref{subsec: data-efficiency}; the headline figure is~\cref{fig: fisher subsample} in the main paper.

\paragraph{Setup.} For each backbone and each subsample fraction $f\in\{0.5\%, 1\%, 5\%, 10\%, 25\%, 50\%, 100\%\}$ we shuffle the training set with a fixed seed, take the prefix of $\lfloor f\cdot N_{t}/B\rfloor$ batches with batch size $B=64$, and accumulate the per-sample squared gradients on that prefix to obtain $\vv_{t}^{(f)}$. The per-sample gradients are materialized in chunks of $8$ samples via \texttt{jax.vmap} so that peak memory remains $\gO(\text{chunk}\cdot|\vtheta|)$ rather than $\gO(B\cdot|\vtheta|)$. We then plug $\{\vv_{t}^{(f)}\}_{t=1}^{8}$ into EpiMer at the optimum with $k=32$, $\alpha=1.0$, and \texttt{frechet\_sum} aggregator from~\cref{appx: alpha-sweep} to evaluate the merged backbone on all eight tasks. The same per-task classifier heads are used as in~\cref{tab: performance}.

\subsection{Computational Cost}
\label[appendix]{appx: computational-cost}

\begin{table}[!t]
    \centering
    \caption{End-to-end merge wall-clock time at the headline setting $k=32$ on a single RTX~3090~Ti, reported as the median of at least four repeat runs on the eight-task CLIP-ViT benchmark. AM/TA, TIES, and Fisher-weighted Averaging operate in the full parameter space and do not trigger the spectral JIT path. TSV-M and EpiMer both construct the per-task tagged basis of~\cref{eq: subspace-basis} and therefore share the JIT-compilation cost of the per-leaf SVD kernel; EpiMer's per-leaf $p\times p$ matrix solve with $p=kT=256$ adds only $3.7$--$11.1$~seconds on top of that shared baseline. The per-task empirical Fisher diagonal used by Fisher Averaging and EpiMer is accumulated once per backbone and amortized across every subsequent merge attempt.}
    \label{tab: computational-cost}
    \resizebox{\linewidth}{!}{%
    \begin{tabular}{@{}lccc@{}}
    \toprule
    \textbf{Method} & \textbf{ViT-B/32} & \textbf{ViT-B/16} & \textbf{ViT-L/14} \\
    \midrule
    AM / Task Arithmetic          & $0.9$\,s  & $0.9$\,s  & $1.3$\,s   \\
    TIES-Merging                  & $3.3$\,s  & $3.4$\,s  & $6.1$\,s   \\
    Fisher-weighted Averaging     & $1.7$\,s  & $1.8$\,s  & $3.8$\,s   \\
    TSV-M                         & $34.2$\,s & $34.2$\,s & $108.7$\,s \\
    \textbf{EpiMer}      & $37.9$\,s & $39.0$\,s & $119.8$\,s \\
    \bottomrule
    \end{tabular}}
\end{table}

\Cref{tab: computational-cost} reports the end-to-end merge wall-clock time at the headline setting $k=32$. The non-spectral baselines (AM/TA, TIES, Fisher Averaging) complete in under seven seconds on every backbone because they operate in the full parameter space with simple per-coordinate aggregation. TSV-M and EpiMer instead pay a fixed per-backbone JIT-compilation cost for the per-leaf SVD kernel that builds the per-task tagged basis $\gB^{(\ell)}$ from~\eqref{eq: subspace-basis}---${\sim}34$~seconds on ViT-B/32 and ViT-B/16, and ${\sim}109$~seconds on ViT-L/14, consistent with L/14's ${\sim}3{\times}$ larger parameter count. On top of that shared floor, EpiMer's per-leaf $p\times p$ matrix solve in~\eqref{eq: frechet-sum} with $p=kT=256$ at $k=32, T=8$ adds only $3.7$, $4.8$, and $11.1$~seconds over TSV-M on ViT-B/32, ViT-B/16, and ViT-L/14 respectively---at most $14\%$ of the corresponding TSV-M baseline. The per-task empirical Fisher diagonal $\vv_{t}$ needs one forward+backward pass per batch, is robust to a tiny data fraction (\cref{subsec: data-efficiency}), and is shared with Fisher-weighted Averaging, so it is amortized once per backbone across all downstream merge attempts and does not enter the headline merge-time comparison.

    {
        \small
        \putbib[references/ieee_abbr,references/juanwu]
    }
\end{bibunit}

\end{document}